\documentclass[twoside]{article}
\usepackage[accepted]{aistats2020}
\usepackage[round]{natbib}

\usepackage[utf8]{inputenc} 
\usepackage[T1]{fontenc} 
\usepackage{hyperref}
\hypersetup{colorlinks,allcolors=black}
\usepackage{url}
\usepackage{setspace}
\usepackage{graphicx}
\usepackage{xcolor}
\usepackage{titlesec}
\usepackage[page]{appendix}
\usepackage{enumerate}
\usepackage{stmaryrd}
\SetSymbolFont{stmry}{bold}{U}{stmry}{m}{n}
\usepackage{bookmark}
\usepackage{enumitem}

\usepackage{lscape}
\usepackage{booktabs}
\usepackage{rotating}
\usepackage{multirow}
\usepackage{longtable}
\usepackage{caption}
\usepackage{subcaption}
\usepackage{float}
\usepackage{tabularx}
\usepackage{ragged2e}
\newcolumntype{Y}{>{\RaggedRight\arraybackslash}X}
\usepackage{pdflscape}
\usepackage{afterpage}

\usepackage{bm}
\usepackage{amsmath}
\usepackage{amssymb}
\usepackage{amsthm}
\usepackage{mathtools}
\usepackage{dsfont}

\usepackage[ruled,vlined]{algorithm2e}
\usepackage{algorithmic}

\newcommand{\diag}{\mathop{\text{diag}}}

\newtheoremstyle{algodesc}{}{}{}{}{\bfseries}{.}{ }{}%
\theoremstyle{algodesc}

\usepackage{textcomp}
\usepackage{sourcecodepro}
\usepackage{listings}
\definecolor{commentgrey}{gray}{0.45}
\definecolor{backgray}{gray}{0.96}
\DeclareMathOperator{\Var}{Var}

\DeclareMathOperator*{\argmin}{arg\,min}
\DeclareMathOperator*{\tr}{tr}
\DeclareMathOperator{\vecto}{vec}

\DeclarePairedDelimiter{\norm}{\lVert}{\rVert}

\newcommand*{\E}{\mathbb E}
\newcommand*{\R}{\mathbb R}

\newcommand*{\bslash}{\mt{\textbackslash}}

\newcommand*{\cond}{\;\ifnum\currentgrouptype=16 \middle\fi|\;}

\newcommand*{\ttilde}{{\raise.17ex\hbox{$\scriptstyle\sim$}}}
\newcommand*{\tensor}[1]{\bm{\mathcal{#1}}}
\newcommand*{\mat}[1]{\mathbf{#1}}

\makeatletter
\newsavebox{\mybox}\newsavebox{\mysim}
\newcommand*{\distas}[1]{%
  \savebox{\mybox}{\hbox{\kern3pt$\scriptstyle#1$\kern3pt}}%
  \savebox{\mysim}{\hbox{$\sim$}}%
  \mathbin{\overset{#1}{\kern\z@\resizebox{\wd\mybox}{\ht\mysim}{$\sim$}}}%
}
\makeatother

\makeatletter
\def\moverlay{\mathpalette\mov@rlay}
\def\mov@rlay#1#2{\leavevmode\vtop{%
   \baselineskip\z@skip \lineskiplimit-\maxdimen
   \ialign{\hfil$\m@th#1##$\hfil\cr#2\crcr}}}
\newcommand*{\charfusion}[3][\mathord]{
  #1{\ifx#1\mathop\vphantom{#2}\fi\mathpalette\mov@rlay{#2\cr#3}}
  \ifx#1\mathop\expandafter\displaylimits\fi}
\makeatother

\newcommand*{\mt}[1]{\text{\normalfont #1}}

\newtheorem{theorem}{Theorem}[section]
\newtheorem{theorem*}{Theorem}

\newtheorem{corollary*}{Corollary}

\newtheorem{proposition*}{Proposition}
\newtheorem{lemma}{Lemma}[section]
\newtheorem{lemma*}{Lemma}

\theoremstyle{definition}

\newtheorem{definition*}{Definition}

\begin{document}

%

%

\twocolumn[

\aistatstitle{The Sylvester Graphical Lasso (SyGlasso)}

\aistatsauthor{Yu Wang \And Byoungwook Jang \And Alfred Hero}

\aistatsaddress{University of Michigan\\wayneyw@umich.edu \And University of Michigan\\bwjang@umich.edu \And University of Michigan\\hero@umich.edu } ]

\begin{abstract}
This paper introduces \textit{the Sylvester graphical lasso} (SyGlasso) that captures multiway dependencies present in tensor-valued data. The model is based on the Sylvester equation that defines a generative model. The proposed model complements the tensor graphical lasso \citep{greenewald2017tensor} that imposes a Kronecker sum model for the inverse covariance matrix by providing an alternative Kronecker sum model that is generative and interpretable. A nodewise regression approach is adopted for estimating the conditional independence relationships among variables. The statistical convergence of the method is established, and empirical studies are provided to demonstrate the recovery of meaningful conditional dependency graphs. We apply the SyGlasso to an electroencephalography (EEG) study to compare the brain connectivity of alcoholic and nonalcoholic subjects. We demonstrate that our model can simultaneously estimate both the brain connectivity and its temporal dependencies.
\end{abstract}

\section{Introduction}
Estimating conditional independence patterns of multivariate data has long been a topic of interest for statisticians. In the past decade, researchers have focused on imposing sparsity on the precision matrix (inverse covariance matrix) to develop efficient estimators in the high-dimensional statistics regime where $n\ll p$. The success of the $\ell_1$-penalized method for estimating multivariate dependencies was demonstrated in \citet{meinshausenbuhlmann06} and \citet{friedmanhastietib08graphicallasso} for the multivariate setting. This has naturally led researchers to generalize these methods to multiway tensor-valued data. Such generalizations are of benefit for many applications, including the estimation of brain connectivity in neuroscience, reconstruction of molecular networks, and detecting anomalies in social networks over time. 

The first generalizations of multivariate analysis to the tensor-variate settings were presented by \citet{dawid1981some}, where the matrix-variate (a.k.a. two-dimensional tensor) distribution was first introduced to model the dependency structures among both rows and columns. \citet{dawid1981some} extended the multivariate setting by rewriting the tensor-variate data as a vectorized (vec) representation of the tensor samples $\tensor{X} \in \R^{m_1 \times \cdots \times m_k}$ and analyzing the overall precision matrix $\mat{\Omega} = \E \big ( \mt{vec}(\tensor{X}) \mt{vec}(\tensor{X})^T \big ) \in \R^{m \times m}$, where $m = \prod_{k=1}^K m_k$. Even for a two-dimensional tensor $\tensor{X} \in \R^{m_1 \times m_2}$, the computation complexity and sample complexity is high since the number of parameters in the precision matrix grows quadratically as $m^2$. Therefore, in the regime of tensor-variate data, unstructured precision matrix estimation has posed challenges due to the large number of samples needed for accurate structure recovery. 

To address the sample complexity challenges, sparsity can be imposed on the precision matrix $\mat{\Omega}$ by using a sparse Kronecker product (KP) or Kronecker sum (KS) decompositions of $\mat{\Omega}$. The earliest and most popular form of sparse structured precision matrix estimation represents $\mat{\Omega}$ as the Kronecker product of smaller precision matrices. \citet{tsiligkaridis2013convergence} and \citet{zhou14} proposed to model the precision matrix as a sparse Kronecker product of the covariance matrices along each mode of the tensor in the form $\bm \Omega =  \bm{\Psi}_1 \otimes \cdots \otimes \bm \Psi_K$. The KP structure on the precision matrix has the nice property that the corresponding covariance matrix is also a KP. \citet{zhou14} provides a theoretical framework for estimating the $\mat{\Omega}$ under KP structure and showed that the precision matrices can be estimated from a single instance under the matrix-variate normal distribution. \citet{lyu2019tensor} extended the KP structured model to tensor-valued data, and provided new theoretical insights into the KP model. An alternative, called the Bigraphical Lasso, was proposed by \citet{kalaitzis2013bigraphical} to model conditional dependency structures of precision matrices by using a Kronecker sum representation $\bm \Omega = \bm \Psi_1 \oplus \bm \Psi_2 = (\bm \Psi_1 \otimes \mat I) + (\mat I \otimes \bm \Psi_2)$. On the other hand, \citet{rudelsonzhou17errinvardependent} and \cite{parketal17_kroneckersum} studied the KS structure on the covariance matrix $\bm \Sigma = \mat A \oplus \mat B$ which corresponds to errors-in-variables models. More recently, \citet{greenewald2017tensor} proposed a model that generalized the KS structure to model tensor-valued data, called the TeraLasso. As shown in their paper, compared to the KP structure, KS structure on the precision matrix leads to a non-separable covariance matrix that provides a richer model than the KP structure.

\textbf{KP vs KS:} The KP model admits a simple stochastic representation as $\mat{X}=\mat{C}^{-1}\mat{Z}\mat{D}^{-1}$, where $\mat{A}=\mat{C}\mat{C}^T, \mat{B}=\mat{D}\mat{D}^T$, and $\mat{Z}$ is white Gaussian. It can be shown using properties of KP that $X \sim \mathcal{N}(0,(\mat{A} \otimes \mat{B})^{-1})$.  Unlike the KP model, the KS model does not have a simple stochastic representation. From another perspective, the Kronecker structures can be characterized by the product graphs of the individual components. \citet{kalaitzis2013bigraphical} first motivated the KS structure on the precision matrix by relating $(\mat{\Psi}_1 \oplus \cdots \oplus \mat{\Psi}_K)$ to the associated Cartesian product graph. Thus, the overall structure of $\mat{\Omega}$ naturally leads to an interpretable model that brings the individual components together. The KP, however, corresponds to the direct tensor product of the individual graphs and leads to a denser dependency structure in the precision matrix \citet{greenewald2017tensor}. 

\textbf{The Sylvester Graphical Lasso (SyGlasso):} We propose a \textit{Sylvester-structured graphical model} to estimate precision matrices associated with tensor data. Similar to the KP- and KS-structured graphical models, we simultaneously learn $K$ graphs along each mode of the tensor data. However, instead of a KS or KP model for the precision matrix, the Sylvester structured graphical model uses a KS model for the square root factor of the precision matrix. The model is estimated by joint sparse regression models that impose sparsity on the individual components $\bm\Psi_k$ for $k=1, \dots, K$. The Sylvester model reduces to a squared KS representation for the precision matrix $\mat{\Omega} = (\mat{\Psi}_1 \oplus \cdots \oplus \mat{\Psi}_K)^2$, which is motivated by a stochastic representation of multivariate data with such a precision matrix. SyGlasso is the first KS-based graphical lasso model that admits a stochastic representation (i.e., Sylvester). Thus, our proposed SyGlasso puts the KS representations on similar ground as the KP representations in terms of interpretablility.

\subsection*{Notations}
We adopt the notations used by \citet{kolda2009tensor}. A $K$-th order tensor is denoted by boldface Euler script letters, e.g, $\tensor{X} \in \R^{m_1 \times \dots \times m_K}$. $\tensor{X}$ reduces to a vector for $K=1$ and to a matrix for $K=2$. The $(i_1,\dots, i_K)$-th element of $\tensor{X}$ is denoted by $\tensor{X}_{i_1,\dots, i_K}$, and we define the vectorization of $\tensor{X}$ to be $\vecto(\tensor{X}) := (\tensor{X}_{1,1,\dots,1},\tensor{X}_{2,1,\dots,1},\dots,\tensor{X}_{m_1,1,\dots,1},\tensor{X}_{1,2,\dots,1},$ $\dots,\tensor{X}_{m_1,m_2,\dots,m_k})^T \in \R^m$ with $m=\prod_{k=1}^K m_k$.

There are several tensor algebra concepts that we recall. A fiber is the higher order analogue of the row and column of matrices. It is obtained by fixing all but one of the indices of the tensor, e.g., the mode-$k$ fiber of $\tensor{X}$ is $\tensor{X}_{i_1,\dots,i_{k-1},:,i_{k+1},\dots,i_K}$. Matricization, also known as unfolding, is the process of transforming a tensor into a matrix. The mode-$k$ matricization of a tensor $\tensor{X}$, denoted by $\tensor{X}_{(k)}$, arranges the mode-$k$ fibers to be the columns of the resulting matrix. It is possible to multiply a tensor by a matrix -- the $k$-mode product of a tensor $\tensor{X} \in \R^{m_1 \times \dots \times m_K}$ and a matrix $\mat{A} \in \R^{J \times m_k}$, denoted as $\tensor{X} \times_k \mat{A}$, is of size $m_1 \times \dots \times m_{k-1} \times J \times m_{k+1} \times \dots m_k$. Its entry is defined as $(\tensor{X} \times_k \mat{A})_{i_1,\dots,i_{k-1},j,i_{k+1},\dots,i_K} := \sum_{i_k=1}^{m_k} \tensor{X}_{i_1,\dots,i_K} A_{j,i_k}$. In addition, for a list of matrices $\{\mat{A}_1,\dots,\mat{A}_K\}$ with $\mat{A}_k \in \R^{m_k \times m_k}$, $k=1,\dots,K$, we define $\tensor{X} \times \{\mat{A}_1,\dots,\mat{A}_K\} := \tensor{X} \times_1 \mat{A}_1 \times_2 \dots \times_K \mat{A}_K$. Lastly, we define the $K$-way Kronecker product as $\bigotimes_{k=1}^K \bm\Psi_k = \bm\Psi_1 \otimes \cdots \otimes \bm\Psi_K$, and the equivalent notation for the Kronecker sum as $\bigoplus_{k=1}^K \bm\Psi_k = \bm\Psi_1 \oplus \dots \oplus \bm\Psi_K = \sum_{k=1}^K \mat I_{[d_{1:k-1}]} \otimes \bm\Psi_k \otimes \mat I_{[d_{k+1:K}]}$, where $\mat I_{[d_{k:\ell}]} = \mat I_{d_k} \otimes \dots \otimes \mat I_{d_\ell}$.

\section{Sylvester Graphical Lasso}
Let a random tensor $\tensor{X} \in \R^{m_1 \times \dots \times m_K}$ be generated by the following representation:
\begin{equation}\label{eqn:tensor_sylvester}
    \tensor{X} \times_1 \mat{\Psi}_1 + \cdots + \tensor{X} \times_K \mat{\Psi}_K = \tensor{T},
\end{equation} 
where $\mat{\Psi}_k \in \R^{m_k \times m_k}, k=1,\dots,K$ are sparse symmetric positive definite matrices and $\tensor{T}$ is a random tensor of the same order as $\tensor{X}$. Equation \eqref{eqn:tensor_sylvester} is known as the Sylvester tensor equation. The equation often arises in finite difference discretization of linear partial equations in high dimension \citep{bai2003hermitian} and discretization of separable PDEs \citep{kressner2010krylov,grasedyck2004existence}. When $K=2$ it reduces to the Sylvester matrix equation $\mat{\Psi_1} \mat{X} + \mat{X} \mat{\Psi_2}^T = \mat{T}$ which has wide application in control theory, signal processing and system identification (see, for example \citet{datta2017cocolasso} and references therein).

It is not difficult to verify that the Sylvester representation \eqref{eqn:tensor_sylvester} is equivalent to the following system of linear equations:
\begin{equation}\label{eqn:linear_tensor_sylvester}
  \left( \bigoplus_{k=1}^K \bm\Psi_k \right ) \vecto(\tensor{X}) = \vecto(\tensor{T}),
\end{equation}
If $\tensor{T}$ is a random tensor such that $\vecto(\tensor{T})$ has zero mean and identity covariance, it follows from \eqref{eqn:linear_tensor_sylvester} that any $\tensor{X}$ generated from the stochastic relation \eqref{eqn:tensor_sylvester} satisfies $\E \vecto(\tensor{X}) = \mat{0}$ and $\mat{\Sigma} = \mat{\Omega}^{-1} := \E \vecto(\tensor{X}) \vecto(\tensor{X})^T = \left ( \bigoplus_{k=1}^K \mat{\Psi}_k \right)^{-2}$. In particular, when $\vecto(\tensor{T}) \sim \mathcal{N}(\mat{0},\mat{I}_m)$, we have that $\vecto(\tensor{X}) \sim \mathcal{N}\left (\mat{0}, \left ( \bigoplus_{k=1}^K \mat{\Psi}_k \right)^{-2} \right)$.

This paper proposes a procedure for estimating $\mat{\Omega}$ with 
$N$ independent copies of the tensor data $\{\tensor{X}^i\}_{i=1}^N$ that are generated from \eqref{eqn:tensor_sylvester}. For the rest of the paper, we assume that the last mode of the data tensor corresponds to the observations mode. For example, when $K=2$, $\tensor{X} \in \R^{m_1 \times m_2 \times N}$ is the matrix-variate data with $N$ observations. Our goal is to estimate the $K$ precision matrices $\{ \mat{\Psi_k} \}_{k=1}^K$ each of which describes the conditional independence of $k$-th data dimension. The resulting precision matrix is $\mat{\Omega} = \left ( \bigoplus_{k=1}^K \mat{\Psi}_k \right)^2$. By rewriting \eqref{eqn:linear_tensor_sylvester} element-wise, we first observe that
\begin{equation}
\label{eqn:elementwise_tensor_sylvester}
\begin{aligned}
    & \left( \sum_{k=1}^K (\mat{\Psi}_k)_{i_k,i_k} \right) \tensor{X}_{i_{[1:K]}} \\
    & = -\sum_{k=1}^K \sum_{j_k \neq i_k} (\mat{\Psi}_k)_{i_k,j_k} \tensor{X}_{i_{[1:k]},j_k,i_{[k+1:K]}} + \tensor{T}_{i_{[1:K]}}.
\end{aligned}
\end{equation} 
Note that the left-hand side of \eqref{eqn:elementwise_tensor_sylvester} involves only the summation of the diagonals of the $\mat{\Psi}$'s and the right-hand side is composed of columns of $\bm\Psi$'s that exclude the diagonal terms. Equation \eqref{eqn:elementwise_tensor_sylvester} can be interpreted as an autogregressive model relating the $(i_1,\dots,i_K)$-th element of the data tensor (scaled by the sum of diagonals) to other elements in the fibers of the data tensor. The columns of $\mat{\Psi}'s$ act as regression coefficients. The formulation in \eqref{eqn:elementwise_tensor_sylvester} naturally leads us to consider a pseudolikelihood-based estimation procedure \citep{besag1977efficiency} for estimating $\bm\Omega$. It is known that inference using pseudo-likelihood is consistent and enjoys the same $\sqrt{N}$ convergence rate as the MLE in general \citep{varin2011overview}. This procedure can also be more robust to model misspecification. Specifically, we define the sparse estimate of the underlying precision matrices along each axis of the data as the solution of the following convex optimization problem:
\begin{equation}
\label{eqn:objective}
  \begin{aligned}
    & \min_{\substack{\mat{\Psi}_k \in \R^{m_k \times m_k}\\k=1,\dots K}} -N \sum_{i_1,\dots,i_K} \log \tensor{W}_{i_{[1:K]}} \\ 
    & \qquad + \frac{1}{2} \sum_{i_1,\dots,i_K} \norm{(I) + (II)}_2^2 + \sum_{k=1}^K P_{\lambda_k}(\mat{\Psi}_k).
  \end{aligned} \
\end{equation}
where $P_{\lambda_k}(\cdot)$ is a penalty function indexed by the tuning parameter $\lambda_k$ and 
\begin{align*}
  (I) & = \tensor{W}_{i_{[1:K]}}\tensor{X}_{i_{[1:K]}} \\
  (II) & = \sum_{k=1}^K \sum_{j_k \neq i_k} (\mat{\Psi}_k)_{i_k,j_k} \tensor{X}_{i_{[1:k]},j_k,i_{[k+1:K]}},
\end{align*}
with $\tensor{W}_{i_{[1:K]}} := \sum_{k=1}^K (\mat{\Psi}_k)_{i_k,i_k}$. Here we focus on the $\ell_1$-norm penalty, i.e., $P_{\lambda_k}(\mat{\Psi}_k) = \lambda_k \norm{\mat{\Psi}_k}_{1,\text{off}}$.

The optimization problem \eqref{eqn:objective} can be put into the following matrix form:
\begin{equation}\label{eqn:objective_matrix}
    \begin{aligned}
    \min_{\substack{\mat{\Psi}_k \in \R^{m_k \times m_k}\\ k=1,\dots K}} 
    & -\frac{N}{2} \log|(\text{diag}(\mat{\Psi}_1) \oplus \dots \oplus \text{diag}(\mat{\Psi}_K))^2| \\ \nonumber
    + & \frac{N}{2} \tr(\mat{S}(\mat{\Psi}_1 \oplus \dots \oplus \mat{\Psi}_K)^2) + \sum_{k=1}^K P_{\lambda_k}(\mat{\Psi}_k) \nonumber
    \end{aligned}
\end{equation}
where $\diag(\mat{\Psi}_k) \in \R^{m_k \times m_k}$ is a matrix of the diagonal entries of $\mat{\Psi}_k$ and $\mat{S} \in \R^{m \times m}$ is the sample covariance matrix, i.e., $\mat{S}=\frac{1}{N} \vecto(\tensor{X})^T \vecto(\tensor{X})$. Note that the pseudolikelihood above approximates the $\ell_1$-penalized Gaussian negative log-likelihood in the log-determinant term by including only the Kronecker sum of the diagonal matrices instead of the Kronecker sum of the full matrices. Further discussion of pseudolikelihood- and likelihood-based approaches for (inverse) covariance estimations can be found in \citet{khare2015convex}.

We also note that when $K=1$ the objective \eqref{eqn:objective} reduces to the objective of the CONCORD estimator \citep{khare2015convex}, and is similar to those of SPACE \citep{peng2009partial} and Symmetric lasso \citep{friedman2010applications}. Our framework is a generalization of these methods to higher order tensor-valued data, when the Sylvester representation \eqref{eqn:tensor_sylvester} holds.

\vspace{\baselineskip}

\noindent\textbf{Remark:} In our formulation $\mat{\Omega}=(\bigoplus_{k=1}^K \mat \Psi_k)^2$ does not uniquely determine $\{\mat{\Psi}_k\}_{k=1}^K$ due to the trace ambiguity: scaled identity factors can be added to/subtracted from the $\mat{\Psi}_k's$ without changing the matrix $\bm\Omega$. To address this non-identifiability, we rewrite the overall precision matrix $\mat{\Omega}$ as
\begin{equation*}
\begin{aligned}
  \mat{\Omega}  = \left( \bigoplus_{k=1}^K \mat{\Psi}_k \right )^2
   = \left( \bigoplus_{k=1}^K \mat{\Psi}_k^{\mt{off}} + \bigoplus_{k=1}^K \text{diag}(\mat{\Psi}_k) \right )^2,
\end{aligned}
\end{equation*} where $\mat{\Psi}_k^{\text{off}}=\mat{\Psi}_k-\text{diag}(\mat{\Psi}_k)$, and estimate the off-diagonal entries $\mat{\Psi}_k^{\text{off}}$ and $\bigoplus_{k=1}^K \text{diag}(\mat{\Psi}_k)$ separately. This allows us to reconstruct the overall precision matrix $\mat{\Omega}$ when $\bm\Psi_k^{\text{off}}$ is penalized with an $\ell_1$ penalty.

\subsection{Estimation of the graphical model}
Let $Q_N(\tensor{W},\{\mat{\Psi}_k^{\text{off}}\}_{k=1}^K)$ denote the objective function in \eqref{eqn:objective}, where $\tensor{W}=\bigoplus_{k=1}^K \text{diag}(\mat{\Psi}_k)$. We adopt a convergent alternating minimization approach \citep{khare2014convergence} that cycles between optimizing $\mat{\Psi}_k$ and $\tensor{W}$ while fixing other parameters.
In particular, for $1 \leq k \leq K$, $1 \leq i_k < j_k \leq m_k$, define
\begin{equation}
\label{eq:sylvester_tensor_update}
\begin{aligned}
    T_{i_kj_k}(\mat{\Psi}_k^{\text{off}}) & = \argmin_{\substack{(\Tilde{\mat{\Psi}}_l)_{m,n}=(\mat{\Psi}_l)_{m,n} \\ \forall (l,m,n) \neq (k,i_k,j_k)}} 
    Q_N(\Tilde{\tensor{W}},\{\Tilde{\mat{\Psi}}_k^{\text{off}}\}_{k=1}^K)\\
    T(\tensor{W}) & = \quad \; \argmin_{\substack{\Tilde{\mat{\Psi}}_k^{\text{off}}=\mat{\Psi}_k^{\text{off}} \\ \forall k}} 
    \quad \; \; Q_N(\Tilde{\tensor{W}},\{\Tilde{\mat{\Psi}}_k^{\text{off}}\}_{k=1}^K).
\end{aligned}
\end{equation}
For each $(k,i_k,j_k)$, $T_{i_kj_k}(\mat{\Psi}_k^{\text{off}})$ 
updates the $(i_k,j_k)$-th entry with the minimizer of $Q_N(\tensor{W},\{\mat{\Psi}\}_{k=1}^K)$ with respect to $(\mat{\Psi}_k)_{i_kj_k}^{\text{off}}$ holding all other variables constant. Similarly, $T(\tensor{W})$ updates $\tensor{W}_{i_{[1:K]}}$ with the solution of $\min Q_N(\tensor{W},\{\mat{\Psi}\}_{k=1}^K)$ with respect to $\tensor{W}_{i_{[1:K]}}$ holding all other variables constant. The closed form updates $T_{i_kj_k}(\mat{\Psi}_k^{\text{off}})$ and $T(\tensor{W})$ are detailed in Appendix \ref{sec:sylvester_teralasso_derivation}.

\begin{algorithm}
\caption{Nodewise SyGlasso}
\label{alg:nodewise_tensor_lasso}
  \SetAlgoLined
  \KwIn{Standardized data $\tensor{X}$, penalty parameter $\lambda_k$}
  \KwOut{$\{\hat{\mat{\Psi}}_k\}_{k=1}^K$, $\hat{\mat{\Omega}}=\left( \bigoplus_{k=1}^K \hat{ \mat{\Psi}}_k \right )^2$}
  Initialize $\{\hat{\mat{\Psi}}_k^{(0)}\}_{k=1}^K$, $\hat{\mat{\Omega}}^{(0)}=\left( \bigoplus_{k=1}^K \hat{ \mat{\Psi}}_k^{(0)} \right )^2$ \\
  \While{not converged}{
  \texttt{\#} \textit{Update off-diagonal elements}\;
    \For{$k \leftarrow 1,\dots,K$}{
      \For{$i_k \leftarrow 1,\dots,m_k-1$}{
        \For{$j_k \leftarrow i_k+1,\dots,m_k$}{
          $(\hat{\mat{\Psi}}_k^{\text{(t+1)}})_{i_k,j_k} \leftarrow (T_{i_k,j_k}(\mat{\Psi}_k^{\text{(t)}}))_{i_k,j_k}$\;
          \hfill from \eqref{eqn:update_offdiag} in Appendix \ref{subsec:derivation_offdiag}
        }
      }
    } 
    \texttt{\#} \textit{Update diagonal elements}\;
    $\hat{\tensor{W}}^{\text{(t+1)}} \leftarrow T(\tensor{W}^{\text{(t)}})$ from \eqref{eqn:update_diag} in Appendix \ref{subsec:derivation_diag}
  } 
\end{algorithm}

\section{Large Sample Properties}\label{thm}
We show that under suitable conditions, the Sylvester graphical lasso (SyGlasso) estimator (Algorithm \ref{alg:nodewise_tensor_lasso}) achieves both model selection consistency and estimation consistency. As in other studies \citep{khare2015convex, peng2009partial}\footnote{When $K=1$ it is possible to relax this assumption to require only accurate estimates of the diagonals, see \citet{khare2015convex, peng2009partial} for details.}, for the convergence analysis we make standard assumptions that the diagonal of $\mat{\Omega}$ is known. We analyze the theoretical properties of the SyGlasso under the assumption that $\tensor{W}$ is given. In practice, we can estimate $\tensor{W}$ using Algorithm \ref{alg:nodewise_tensor_lasso}, and if the diagonals of each individual $\mat{\Psi}_k$ are desired, we can incorporate any available prior knowledge of the variation along each data dimension.

We estimate $\{\mat{\Psi}_k^{\text{off}}\}_{k=1}^K$ by solving the following $\ell_1$ penalized problem:
\begin{equation}
    \min_{\bm{\beta}} L_N \Big(\tensor{W},\bm{\beta},\tensor{X}\Big) + \sum_{k=1}^K \lambda_k \norm{\mat{\Psi}_k}_{1,\text{off}},
\end{equation} 
where $L_N \Big(\tensor{W},\bm{\beta},\tensor{X}\Big):=\frac{1}{N}\sum_{s=1}^{N} L\Big(\tensor{W},\bm{\beta},\tensor{X}^s \Big)$, with
\begin{equation}
\begin{aligned}
    L\Big(\tensor{W},\bm{\beta},\tensor{X}^s \Big) & = - N \sum_{i_{[1:K]}} \log \tensor{W}_{i_{[1:K]}} \\ 
    & \qquad \qquad + \frac{1}{2} \sum_{i_1,\dots,i_K} ((I) + (II))^2. \\ 
\end{aligned}
\vspace{-10pt}
\end{equation} 
where
\begin{align*}
  & (I)  = \tensor{W}_{i_{[1:K]}}\tensor{X}_{i_{[1:K]}} \\
  & (II)  = \sum_{k=1}^K \sum_{j_k \neq i_k} (\mat{\Psi}_k)_{i_k,j_k} \tensor{X}_{i_{[1:k-1]},j_k,i_{[k+1:K]}}\\
  & \bm{\beta}  = ((\mat{\Psi}_1)_{1,2},(\mat{\Psi}_1)_{1,3},\dots,(\mat{\Psi}_1)_{1,m_1},\dots,(\mat{\Psi}_k)_{m_k-1,m_k})^T
\end{align*} and $\bm \beta$ denotes the off-diagonal entries of all $\mat{\Psi}_k's$. 

We first state the regularity conditions needed for establishing convergence of the SyGlasso estimator. Let $\mathcal{A}_{k}:=\{(i,j):(\mat{\Psi}_k)_{i,j} \neq 0, i \neq j\}$ and $q_{k}:=|\mathcal{A}_{k}|$ for $k=1,\dots,K$ be the true edge set and the number of edges, respectively. Let $\mathcal{A} = \cup_{k=1}^K \mathcal{A}_{k}$. We use $\bar{\bm\beta}, \bar{\bm\Omega}, \bar{\tensor{W}}$ to emphasize that they are the true values of the corresponding parameters.

\noindent \textbf{(A1 - Subgaussianity)} The data $\tensor{X}^1,\dots,\tensor{X}^N$ are i.i.d subgaussian random tensors, that is, $\vecto(\tensor{X}^i) \sim \mat{x}$, where $\mat{x}$ is a subgaussian random vector in $\mathbb{R}^p$, i.e., there exist a constant $c>0$, such that for every $\mat{a} \in \mathbb{R}^p$, $\mathbb{E}e^{\mat{a}^T x} \leq e^{c\mat{a}^T \bar{\mat{\Sigma}} \mat{a}}$, and there exist $\rho_j > 0$ such that $\mathbb{E}e^{tx_j^2} \leq K$ whenever $|t| < \rho_j$, for $1 \leq j \leq p$.

\noindent \textbf{(A2 - Bounded eigenvalues)} There exist constants $0 < \Lambda_{\min} \leq \Lambda_{\max} < \infty$, such that the minimum and maximum eigenvalues of $\mat{\Omega}$ are bounded with $\lambda_{\min}(\bar{\mat{\Omega}}) = (\sum_{k=1}^K \lambda_{\max}(\mat{\Psi}_k))^{-2} \geq \Lambda_{\min}$ and $\lambda_{\max}(\bar{\mat{\Omega}}) = (\sum_{k=1}^K \lambda_{\min}(\mat{\Psi}_k))^{-2} \leq \Lambda_{\max}$.

\noindent \textbf{(A3 - Incoherence condition)} There exists a constant $\delta < 1$ such that for $k=1,\dots,K$ and all $(i,j) \in \mathcal{A}_{k}$
\begin{equation*}
    |\bar{L}_{ij,\mathcal{A}_{k}}^{''}(\bar{\tensor{W}},\bar{\bm{\beta}})[\bar{L}_{\mathcal{A}_{k},\mathcal{A}_{k}}^{''}(\bar{\tensor{W}},\bar{\bm{\beta}})]^{-1} \text{sign}(\bar{\bm{\beta}}_{\mathcal{A}_{k}})| \leq \delta,
\end{equation*} where for each $k$ and $1 \leq i < j \leq m_k$, $1 \leq k < l \leq m_k$,
\begin{equation*}
    \bar{L}_{ij,kl}^{''}(\bar{\tensor{W}},\bar{\bm{\beta}}) := E_{\bar{\tensor{W}},\bar{\bm{\beta}}} \Bigg(\frac{\partial^2 L(\tensor{W},\bm{\beta},\tensor{X})}{\partial(\mat{\Psi}_k)_{i,j} \partial(\mat{\Psi}_k)_{k,l}}|_{\tensor{W}=\bar{\tensor{W}},\bm{\beta}=\bar{\bm{\beta}}} \Bigg).
\end{equation*}

Note that conditions analogous to (A3) have been used in \citet{meinshausenbuhlmann06} and \citet{peng2009partial} to establish high-dimensional model selection consistency of the nodewise graphical lasso in the case of $K=1$. \citet{zhao2006model} show that such a condition is almost necessary and sufficient for model selection consistency in lasso regression, and they provide some examples when this condition is satisfied.

Inspired by \citet{meinshausenbuhlmann06} and \citet{peng2009partial} we prove the following properties:
\begin{enumerate} 
    \vspace{-8pt}
    \item Theorem 3.1 establishes estimation consistency and sign consistency for the nodewise SyGlasso restricted to the true support, i.e., $\bm{\beta}_{\mathcal{A}^c}=0$,
    \item Theorem 3.2 shows that no wrong edge is selected with probability tending to one,
    \item Theorem 3.3 establishes consistency result of the nodewise SyGlasso.
\end{enumerate}

\begin{theorem}
  Suppose that conditions (A1-A2) are satisfied. Suppose further that $\lambda_{N,k}=O(\sqrt{\frac{m_k\log p}{N}})$ for all $k$ and $N > O(\max_k q_k m_k \log p)$ as $N \rightarrow \infty$. Then there exists a constant $C(\bar{\bm{\beta}})$, such that for any $\eta>0$, the following hold with probability at least $1-O(\exp(-\eta \log p))$:
  \begin{itemize}
      \item There exists a global minimizer $\hat{\bm{\beta}}_{\mathcal{A}}$ of the restricted SyGlasso problem:
      \begin{equation}\label{eqn:restricted_problem}
          \min_{\bm{\beta}:\bm{\beta}_{\mathcal{A}^c}=0} L_N(\bar{\tensor{W}},\bm{\beta},\tensor{X}) + \sum_{k=1}^K \lambda_k \norm{\mat{\Psi}_k}_{1,\text{off}}.
      \end{equation}
      \item (Estimation consistency) Any solution $\hat{\bm{\beta}}_{\mathcal{A}}$ of \eqref{eqn:restricted_problem} satisfies:
      \begin{equation*}
          \|\hat{\bm{\beta}}_{\mathcal{A}} - \bm{\beta}_{\mathcal{A}}\|_2 \leq C(\bar{\bm{\beta}})\sqrt{K}\max_{k}\sqrt{q_{k}}\lambda_{N,k}.
      \end{equation*}
      \item (Sign consistency) If further the minimal signal strength satisfies: $\min_{(i,j) \in \mathcal{A}_{k}}|(\mat{\Psi}_k)_{i,j}| \geq 2C(\bar{\bm{\beta}})\sqrt{K}\max_{k}\sqrt{q_{k}}\lambda_{N,k}$ for each $k$, then sign($\hat{\bm{\beta}}_{\mathcal{A}_{k}}$)=sign($\bar{\bm{\beta}}_{\mathcal{A}_{k}}$). 
  \end{itemize}
  
\end{theorem}

\begin{theorem}
Suppose that the conditions of Theorem of 3,1 and (A3) are satisfied. Suppose further that $p=O(N^{\kappa})$ for some $\kappa \geq 0$. Then for $\eta>0$, for $N$ sufficiently large, the solution of \eqref{eqn:restricted_problem} satisfies:
  \begin{align*}
      P_{\bar{\tensor{W}},\bar{\bm{\beta}}} & \Big(\max_{(i,j) \in \mathcal{A}_{k}^c} |L_{N,ij}^{\prime}(\bar{\tensor{W}},\hat{\bm{\beta}}_{\mathcal{A}_{k}},\tensor{X})|<\lambda_{N,k}\Big) \\
      & \geq 1 - O(\exp(-\eta \log p))
  \end{align*} for each $k$, where $L_{N,ij}^{\prime} := \partial L_N / \partial (\mat{\Psi}_k)_{ij}$.
\end{theorem}

\begin{theorem}
  Assume the conditions of Theorem 3.2. Then there exists a constant $C(\bar{\bm{\beta}})>0$ such that for any $\eta>0$ the following events hold with probability at least $1 - O(\exp(-\eta \log p))$:
  \begin{itemize}
      \item There exists a global minimizer $\hat{\bm{\beta}}$ to problem \eqref{eqn:objective}.
      \item (Estimation consistency) Any minimizer $\hat{\bm{\beta}}$ of \eqref{eqn:objective} satisfies:
      \begin{equation*}
          \|\hat{\bm{\beta}} - \bm{\beta}\|_2 \leq C(\bar{\bm{\beta}})\sqrt{K}\max_{k}\sqrt{q_{k}}\lambda_{N,k}.          
      \end{equation*}
      \item (Sign consistency) If $\min_{(i,j) \in \mathcal{A}_{k}}|(\mat{\Psi}_k)_{i,j}| \geq 2C(\bar{\bm{\beta}})\max_{k}\sqrt{q_{k}}\lambda_{N,k}$ for each $k$, then sign($\hat{\bm{\beta}}$)=sign($\bar{\bm{\beta}}$).
  \end{itemize}
\end{theorem}
Proofs of the above theorems are given in Appendix \ref{proofs}.

\section{Numerical Illustrations}\label{numeric}
We evaluate the proposed SyGlasso estimator (Algorithm \ref{alg:nodewise_tensor_lasso}) in terms of optimization and graph recovery accuracy. We also compare the graph recovery performance with other models recently proposed for matrix- and tensor-variate precision matrices. We first illustrate the differences among these models by investigating the sparsity pattern of $\bm\Omega$ with $K=3$ modes and $m_k = 4, \forall k$. For simplicity, we generate $\bm\Psi_k$ for $k=1, 2, 3$ as identical $4 \times 4$ precision matrices that follow a one dimensional autoregressive-1 (AR1) process. We recall the KP and KS models:

\begin{figure}[h!] \centering
\begin{subfigure}[t]{0.45\linewidth} \centering
\includegraphics[width=\linewidth]{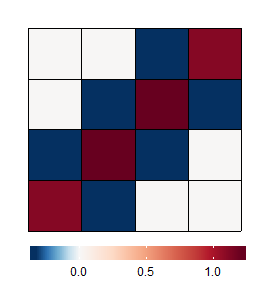}
\caption{$\bm\Psi_k$}
\end{subfigure}
\hspace{7pt}
\begin{subfigure}[t]{0.45\linewidth} \centering
\includegraphics[width=\linewidth]{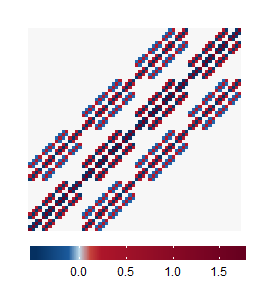}
\caption{KP $\mat{\Omega}$}
\end{subfigure}

\begin{subfigure}[t]{0.45\linewidth} \centering
\includegraphics[width=\linewidth]{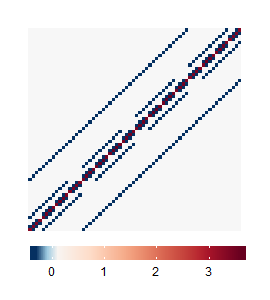}
\caption{KS $\mat{\Omega}$}
\end{subfigure}
\hspace{7pt}
\begin{subfigure}[t]{0.45\linewidth} \centering
\includegraphics[width=\linewidth]{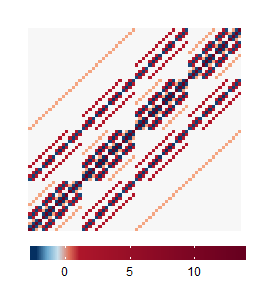}
\caption{SyGlasso $\mat{\Omega}$}
\end{subfigure}
\caption{Comparison of SyGlasso to Kronecker sum (KS) and product (KP) structures. All models are composed of the same components $\mat{\Psi}_k$ for $k=1, 2, 3$ generated as an AR(1) model with $m_k=4$ as shown in (a). The AR(1) components are brought together to create the final $64 \times 64$ precision matrix $\mat{\Omega}$ following (b) the KP structure with $\mat{\Omega} = \bigotimes_{k=1}^3 \mat{\Psi}_k$, (c) the KS structure with $\mat{\Omega}= \bigoplus_{k=1}^3 \mat{\Psi}_k$, and (d) the proposed Sylvester model with $\mat{\Omega}=\left(\bigoplus_{k=1}^3 \mat{\Psi}_k\right)^2$. The KP does not capture nested structures as it simply replicates the individual component with different multiplicative scales. The SyGlasso model admits a precision matrix structure that strikes a balance between KS and KP.}
\label{fig:AR_comparison}
\end{figure}

\textbf{Kronecker Product (KP):} 
The KP model restricts the precision matrix and the covariance matrix to be separable across the $K$ data dimensions and suffers from a multiplicative explosion in the number of edges. As they are separable models and the constructed $\bm\Omega$ corresponds to the direct product of the $K$ graphs, KP is unable to capture more complex nested patterns captured by the KS and SyGlasso models as shown in Figure \ref{fig:AR_comparison} (c) and (d). 

\textbf{Kronecker Sum (KS):} 
The covariance matrix under the KS precision matrix assumption is nonseparable across $K$ data dimensions, and the KS-structured models can be motivated from a maximum entropy point of view. Contrary to the KP structure, the number of edges in the KS structure grows as the sum of the edges of the individual graphs (as a result of Cartesian product of the associated graphs), which leads to a more controllable number of edges in $\bm\Omega$. 

We compare these methods under different model assumptions to explore the flexibility of the proposed SyGlasso model under model mismatch. To empirically assess the efficiency of the proposed model, we generate tensor-valued data based on three different precision matrices. The $\mat{\Psi}_k$'s are generated from one of 1) AR1($\rho$), 2) Star-Block (SB), or 3) Erdos-Renyi (ER) random graph models described in Appendix \ref{sec:simulated_precision_matrix}. 

\begin{figure}[htb] \centering
\begin{subfigure}[t]{0.45\linewidth} \centering
\includegraphics[width=\linewidth]{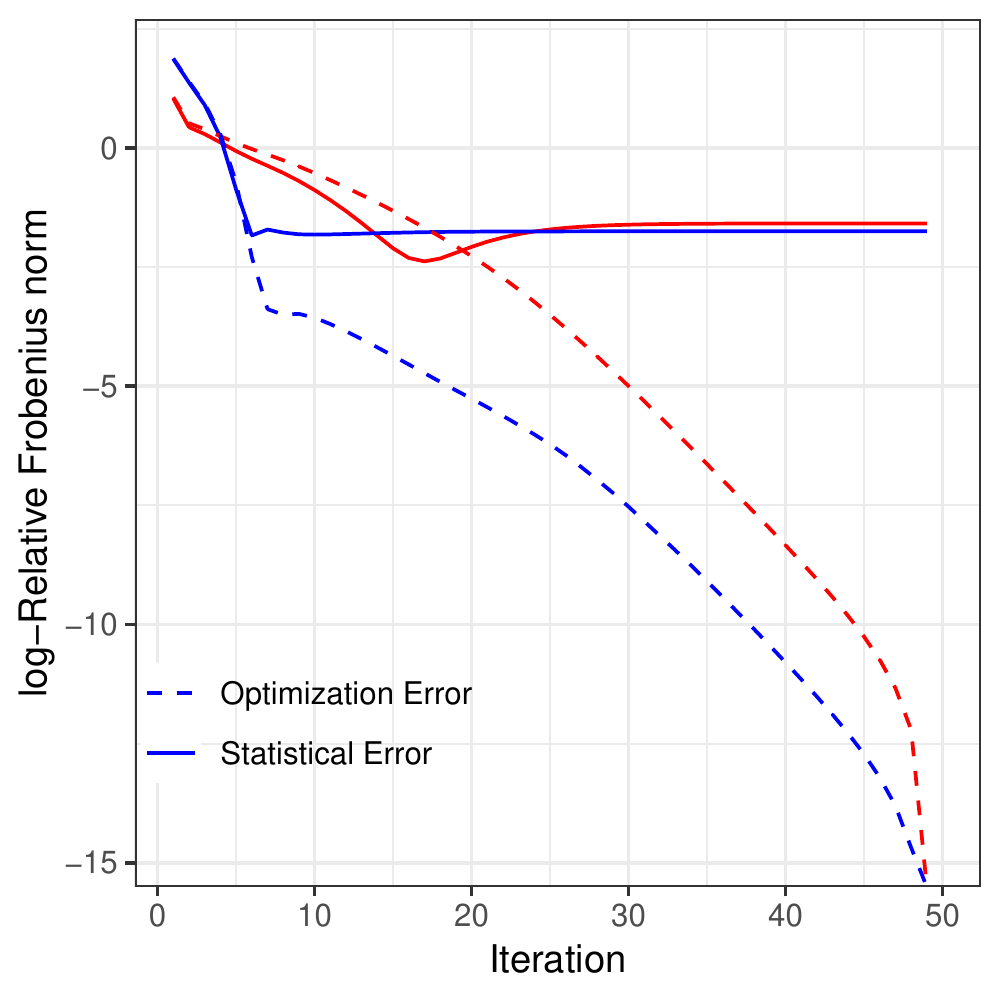}
\caption{SB and AR}
\end{subfigure}
~
\begin{subfigure}[t]{0.45\linewidth} \centering
\includegraphics[width=\linewidth]{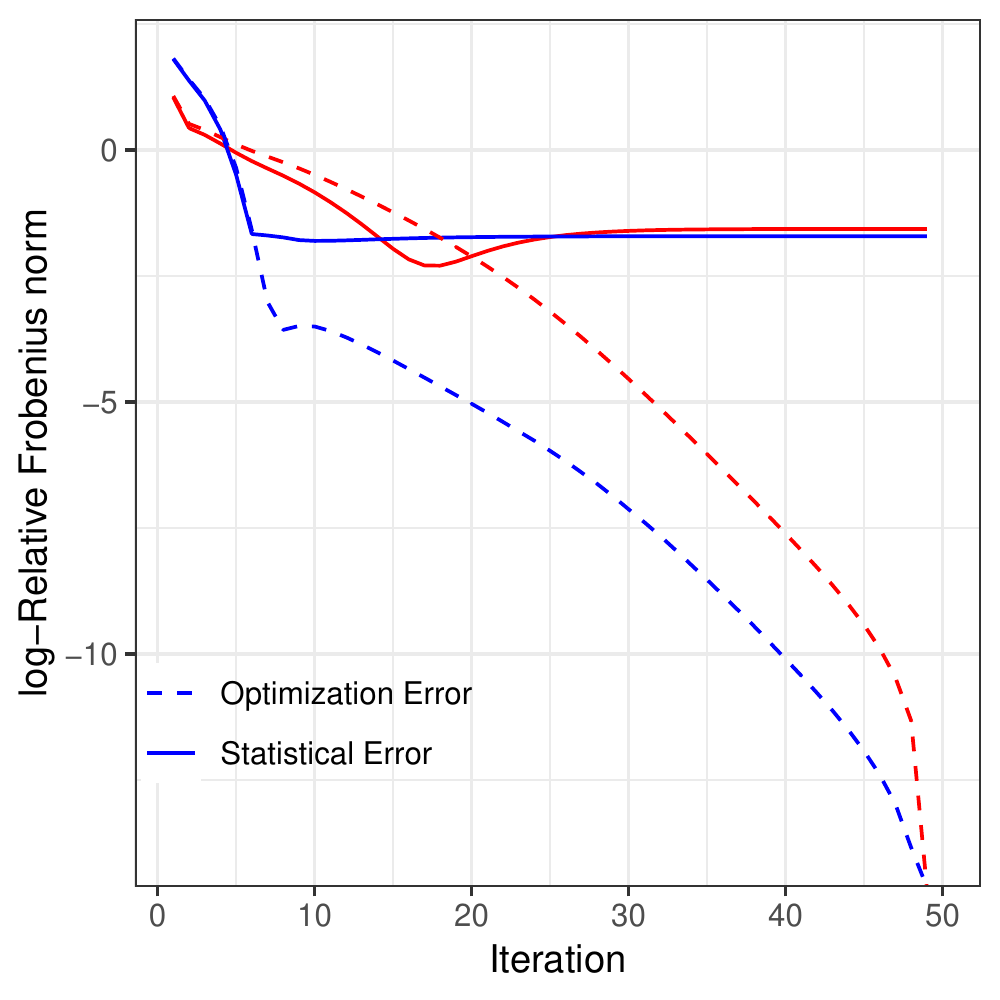}
\caption{SB and ER}
\end{subfigure}
\caption{Performance of the SyGlasso estimator against the number of iterations under different topologies of $\mat{\Psi}_k$'s. The solid line shows the statistical error $\log(\norm{\hat{\mat{\Psi}}_k^{(t)} - \mat{\Psi}_k}_F \bslash \norm{\mat{\Psi}_k}_F)$, and the dotted line shows the optimization error $\log(\norm{\hat{\mat{\Psi}}_k^{(t)} - \hat{\mat{\Psi}}_k}_F \bslash \norm{\hat{\mat{\Psi}}_k}_F)$, where $\hat{\mat{\Psi}}_k$ is the final SyGlasso estimator. The performances of $\mat{\Psi}_1$ and $\mat{\Psi}_2$ are represented by red and blue lines, respectively.}


\label{fig:sim_fnorm}
\end{figure}

We test SyGlasso with $K=2$ under: 1) SB with $\rho=0.6$ and sub-blocks of size $16$ and AR1($\rho=0.6$); 2) SB with $\rho=0.6$ and sub-blocks of size $16$ and ER with $256$ randomly selected edges. In both scenarios we set $m_1=128$ and $m_2=256$ with $10$ samples. Figure \ref{fig:sim_fnorm} shows the iterative optimization performance of Algorithm \ref{alg:nodewise_tensor_lasso}. All the plots for the various scenarios exhibit iterative optimization approximation errors that quickly converge to values below the statistical errors. Note that these plots also suggest that our algorithm can attain linear convergence rates. We also test our method for model selection accuracy over a range of penalty parameters (we set $\lambda_k=\lambda,\forall k$). Figure \ref{fig:sim_roc} displays the sum of false positive rate and false negative rate (FPR+FNR), it suggests that the nodewise SyGlasso estimator is able to fully recover the graph structures for each mode of the tensor data. 

\begin{figure}[htb] \centering
\begin{subfigure}[t]{0.45\linewidth} \centering
\includegraphics[width=\linewidth]{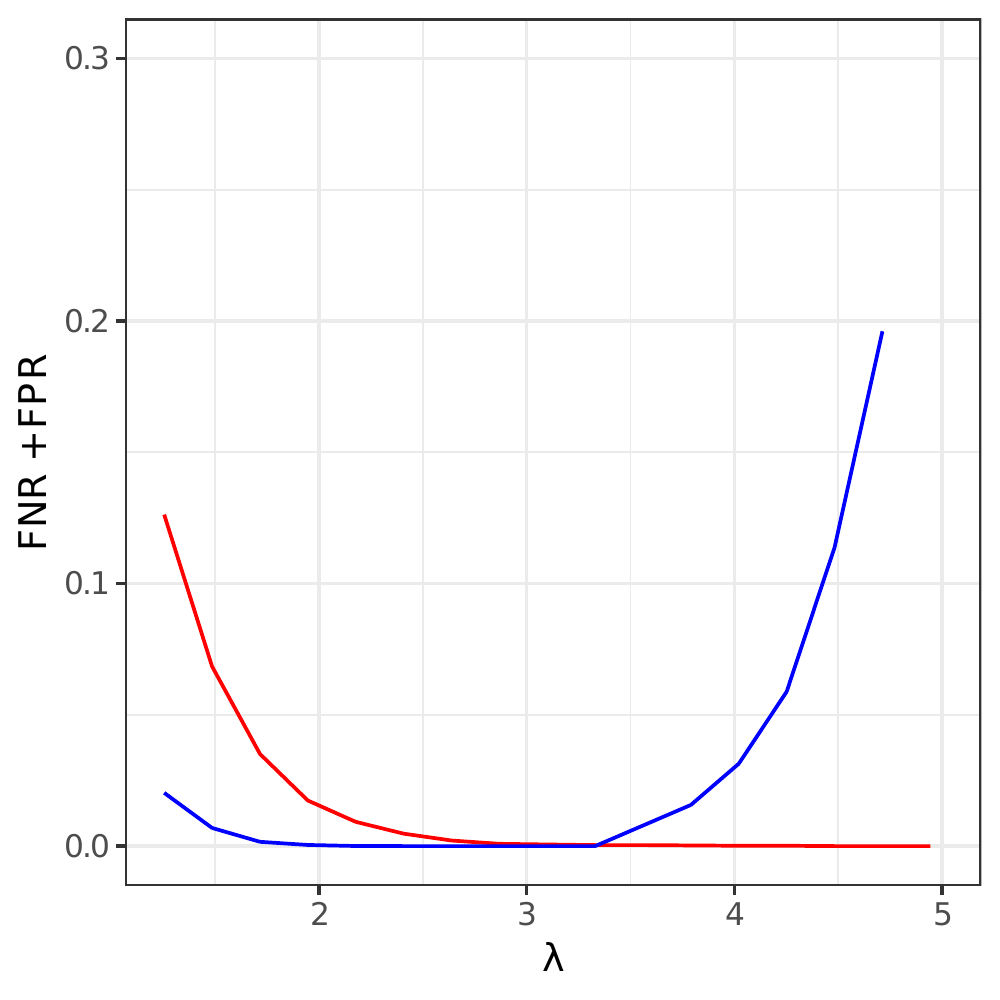}
\caption{SB and AR}
\end{subfigure}
~
\begin{subfigure}[t]{0.45\linewidth} \centering
\includegraphics[width=\linewidth]{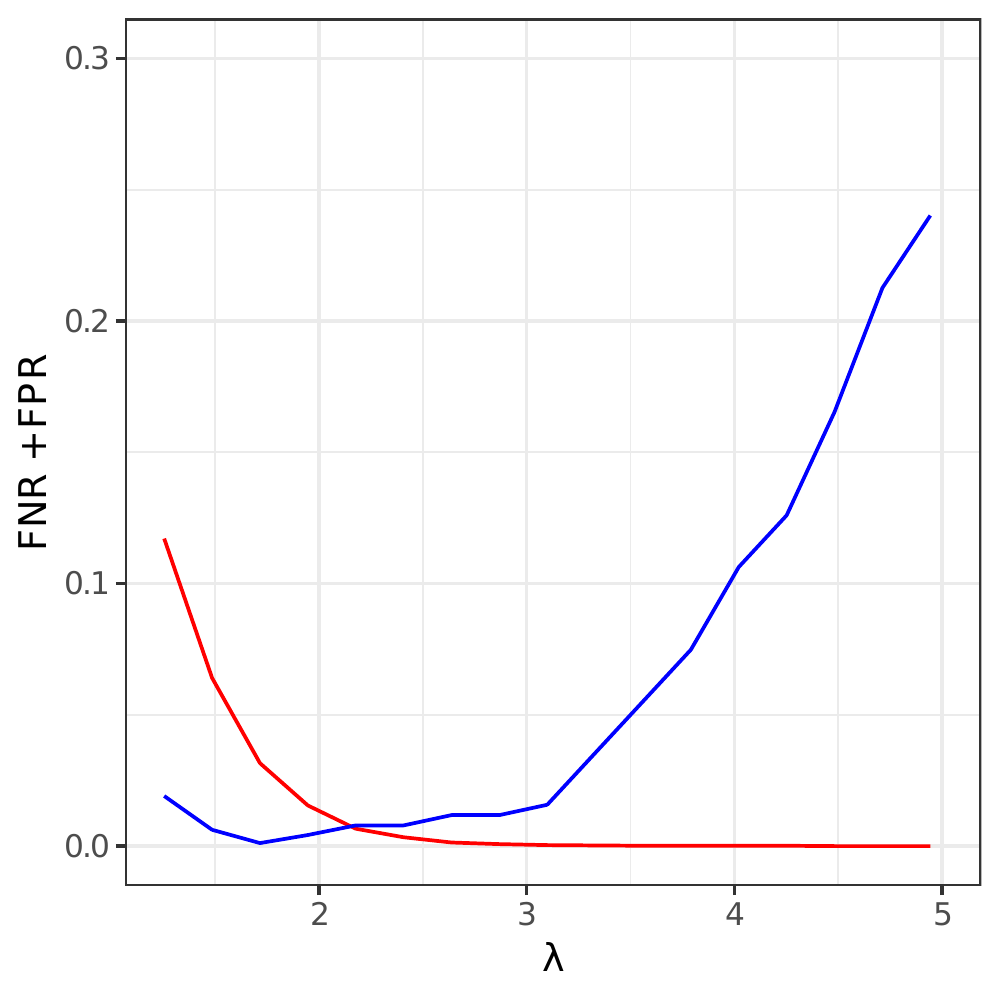}
\caption{SB and ER}
\end{subfigure}
\caption{The performance of model selection measured by FPR + FNR. The performances of $\mat{\Psi}_1$ and $\mat{\Psi}_2$ are represented by red and blue lines, respectively. With an appropriate choice of $\lambda$, the SyGlasso recovers the dependency structures encoded in each $\mat{\Psi}_k$.}
\label{fig:sim_roc}
\end{figure}

We compare the proposed SyGlasso to the TeraLasso estimator \citep{greenewald2017tensor}, and to the Tlasso estimator proposed by \citet{lyu2019tensor} for KP, on data generated using precision matrices $(\mat{\Psi}_1 \oplus \mat{\Psi}_2 \oplus \mat{\Psi}_3)^2$, $\mat{\Psi}_1 \oplus \mat{\Psi}_2 \oplus \mat{\Psi}_3$, and $\mat{\Psi}_1 \otimes \mat{\Psi}_2 \otimes \mat{\Psi}_3$, where $\mat{\Psi}$'s are each $16 \times 16$ ER graphs with $16$ nonzero edges. We use the Matthews correlation coefficient (MCC) to compare model selection performances. The MCC is defined as \citep{matthews1975comparison}
\begin{equation*}
    \text{MCC} = \frac{\text{TP}\times\text{TN}-\text{FP}\times\text{FN}}{\sqrt{(\text{}TP+\text{FP})(\text{TP}+\text{FN})(\text{TN}+\text{FP})(\text{TN}+\text{FN})}},
\end{equation*} where we follow \citet{greenewald2017tensor} to consider each nonzero off-diagonal element of $\mat{\Psi}_k$ as a single edge. 

The results shown in Figure \ref{fig:modelmismatch} indicate that all three estimators perform well when $N=5$, even under model misspecification. In the single sample scenario, the graph recovery performance of each estimator does well under each true underlying data generating process. Note that for data generated using KP, the SyGlasso performs surprisingly well and is comparable to Tlasso. These results seem to indicate that SyGlasso is very robust under model misspecification. The superior performance of SyGlasso under KP model, even with one sample, suggests again that SyGlasso structure has a flavor of both KS and KP structures, as seen in Figure \ref{fig:AR_comparison}. This follows from the observation that  $(\mat{\Psi}_1 \oplus \mat{\Psi}_2)^2 = \mat I_{m_1} \otimes \mat{\Psi}_1^2 + \mat{\Psi}_2^2 \otimes \mat I_{m_2} + 2\mat{\Psi}_1 \otimes \mat{\Psi}_2 = \mat{\Psi}_1^2 \oplus \mat{\Psi}_2^2 + 2\mat{\Psi}_1 \otimes \mat{\Psi}_2$.

\begin{figure}[!htb] \centering
\begin{tabular}{@{}cc@{}}
\qquad \qquad  $N = 1$ & \quad  $N = 5$\\
\rotatebox{90}{\qquad \qquad SyGlasso} \qquad 
\includegraphics[width=0.4\linewidth]{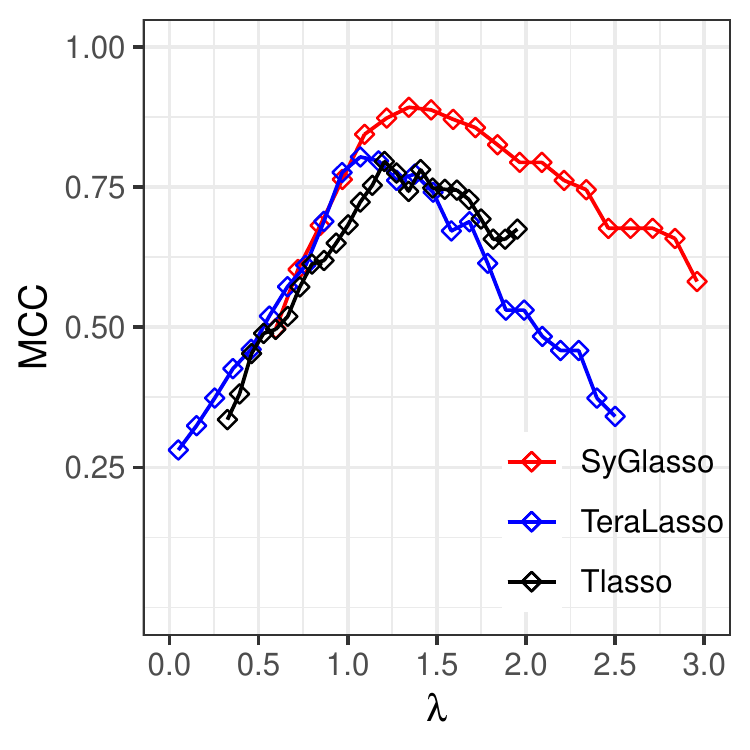}
&
\includegraphics[width=0.4\linewidth]{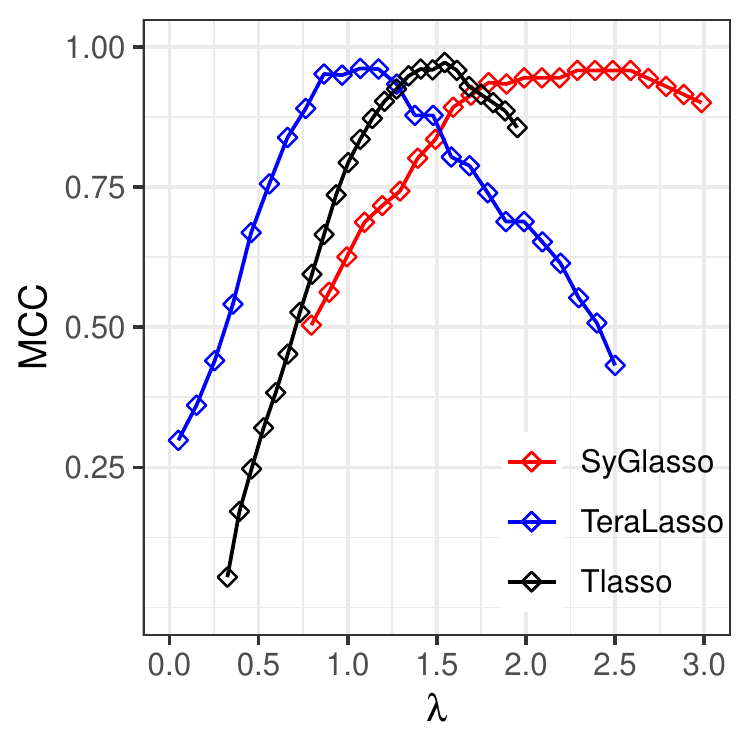}\\
\rotatebox{90}{\qquad \qquad KS} \qquad  
\includegraphics[width=0.4\linewidth]{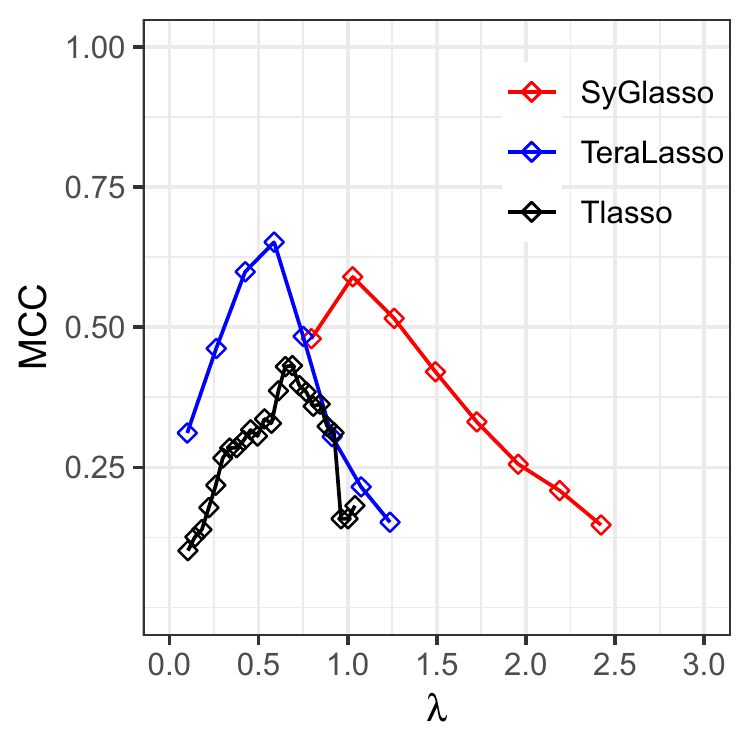}
& 
\includegraphics[width=0.4\linewidth]{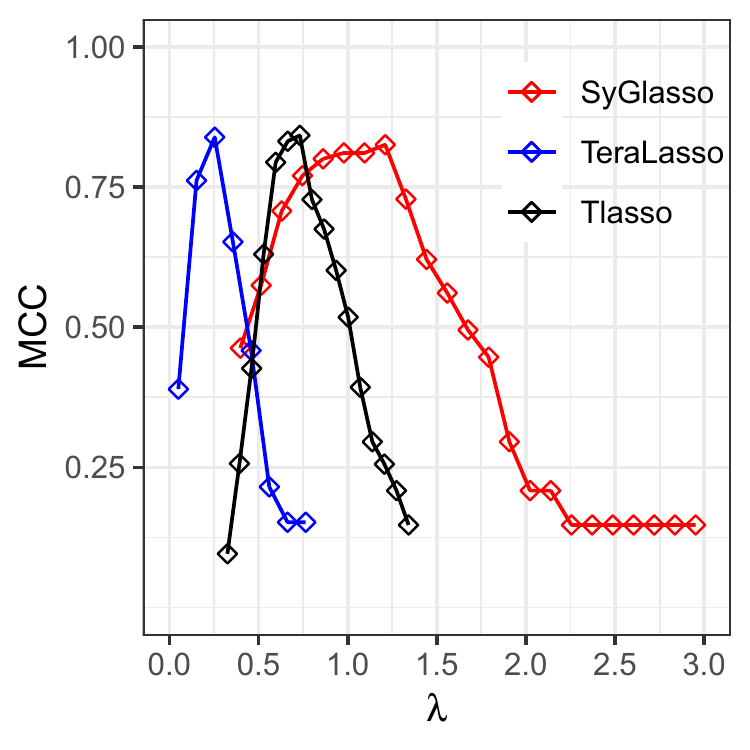}\\
\rotatebox{90}{\qquad \qquad KP} \qquad 
\includegraphics[width=0.4\linewidth]{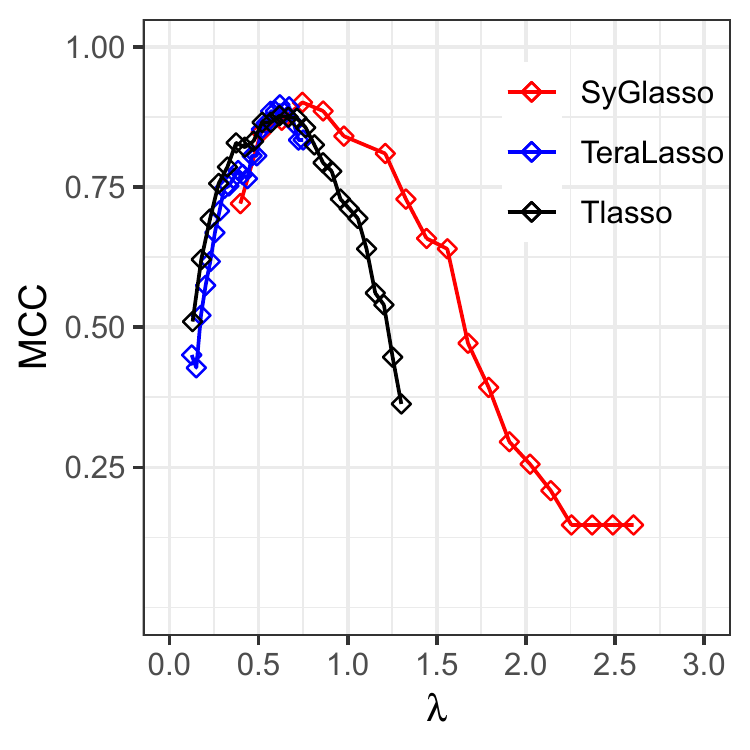}
& 
\includegraphics[width=0.4\linewidth]{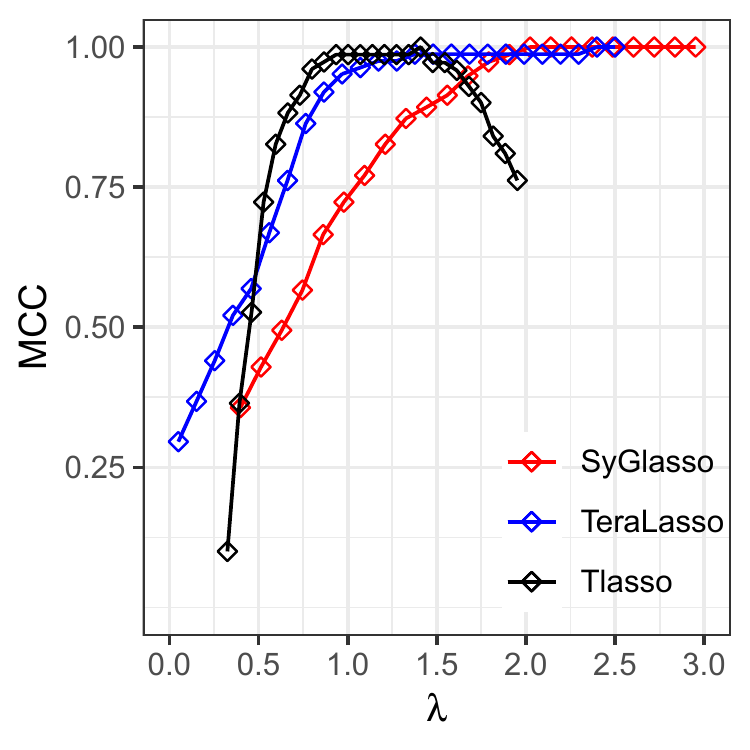}
\end{tabular}
\caption{Performance of SyGlasso, TeraLasso (KS), and Tlasso (KP) measured by MCC under model misspecification. MCC of $1$ represents a perfect recovery of the sparsity pattern in $\mat{\Omega}$, and MCC of $0$ corresponds to random guess. From top to bottom, the synthetic data were generated with the precision matrices from SyGlasso, KS, and KP models. The left column shows the results for a single sample ($N=1$), and the right column shows the results for $N=5$ observations. Note that the SyGlasso has better performance for a single sample (left column) when data is generated from the matched Kronecker model and as good performance for the mismatched Kronecker models.}
\label{fig:modelmismatch}
\end{figure}

\section{EEG Analysis}
We revisit the alcoholism study conducted by \citet{zhang1995event} to explore multiway relationships in EEG measurements of alcoholic and control subjects. Each of 77 alcoholic subjects and 45 control subjects was visually stimulated by either a single picture or a pair of pictures on a computer monitor. Following the analyses of \citet{zhu2016bayesian} and \citet{qiao2019functional}, we focus on the $\alpha$ frequency band (8 - 13 Hz) that is known to be responsible for the inhibitory control of the subjects (see \citet{knyazev2007motivation} for more details). The EEG signals were bandpass filtered with the cosine-tapered window to extract $\alpha$-band signals. Previous Gaussian graphical models applied to such $\alpha$ frequency band filtered EEG data could only estimate the connectivity of the electrodes as they cannot be generalized to tensor valued data. The SyGlasso reveals similar dependency structure as reported in \citet{zhu2016bayesian} and \citet{qiao2019functional} while recovering the chain structure of the temporal relationship.

Specifically, after the band-pass filter was applied, we work with the tensor data $\tensor{X}_{alcoholic}, \tensor{X}_{control} \in \mathbb{R}^{m_{nodes} \times m_{time} \times m_{trial}}$ corresponding to an alcoholic subject and a control subject. We simultaneously estimate $\mat{\Psi}_{node} \in \mathbb{R}^{m_{node} \times m_{node}}$ that encodes the dependency structure among electrodes and $\mat{\Psi}_{time} \in \mathbb{R}^{m_{time}\times m_{time}}$ that shows the relationship among time points that span the duration of each trial. Previous studies consider the average of all trials, for each subject and use the number of subjects as observations to estimate the dependency structures among $64$ electrodes. Instead, we look at one subject at a time and consider different experimental trials as observations. Our analysis focuses on recovering the precision matrices of electrodes and time points, but it can be easily generalized to estimate the dependency structure among trials as well.
\begin{figure}[!thb] \centering
\begin{subfigure}[t]{0.5\linewidth} \centering
\includegraphics[width=\linewidth]{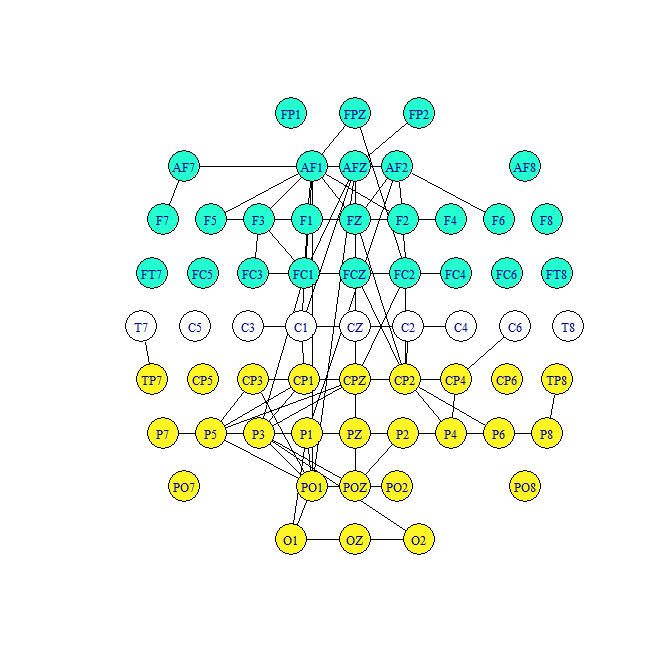}
\caption{Alcoholic subject}
\end{subfigure}
\hspace{-20pt}
\begin{subfigure}[t]{0.5\linewidth} \centering
\includegraphics[width=\linewidth]{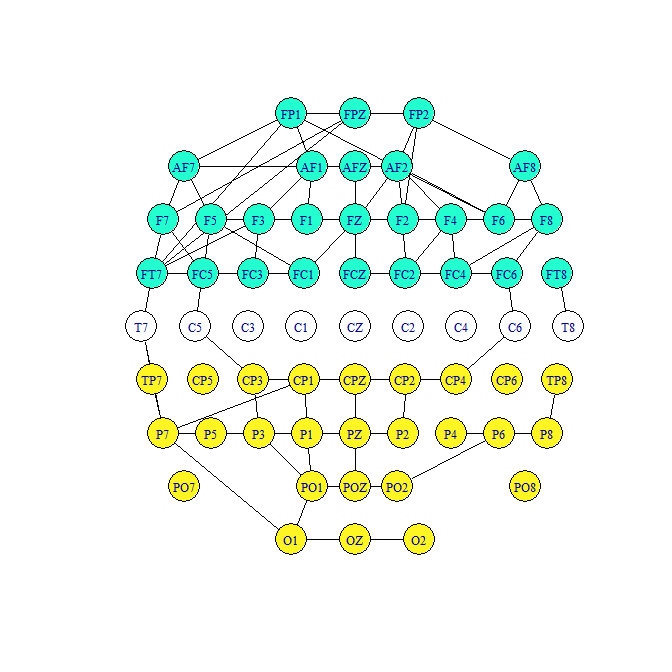}
\caption{Control subject}
\end{subfigure}
\caption{Estimated brain connectivity results from SyGlasso for (a) the alcoholic subject and (b) the control subject. The blue nodes correspond to the frontal region, and the yellow nodes correspond to the parietal and occipital regions. The alcoholic subject has asymmetric brain connections in the frontal region compared to the control subject.}
\label{fig:alcoholism_node_result}
\end{figure}

Figure \ref{fig:alcoholism_node_result} shows the result of the SyGlasso estimated network of electrodes. For comparison, both graphs were thresholded to match 5\% sparsity level. Similar to the findings of \citet{qiao2019functional}, our estimated graph $\bm\Psi_{node}$ for the alcoholic group shows the asymmetry between the left and the right side of the brain compared to the more balanced control group. Our finding is also consistent with the result in \citet{hayden2006patterns} and \citet{zhu2016bayesian} that showed frontal asymmetry of the alcoholic subjects. 

\begin{figure}[!htb] \centering
\begin{subfigure}[t]{0.45\linewidth} \centering
\includegraphics[width=\linewidth]{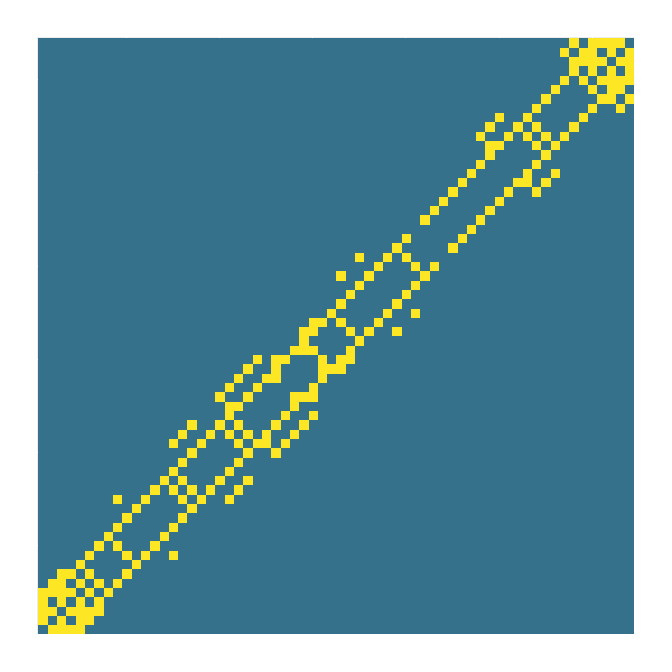}
\caption{Alcoholic subject}
\end{subfigure}
\begin{subfigure}[t]{0.45\linewidth} \centering
\includegraphics[width=\linewidth]{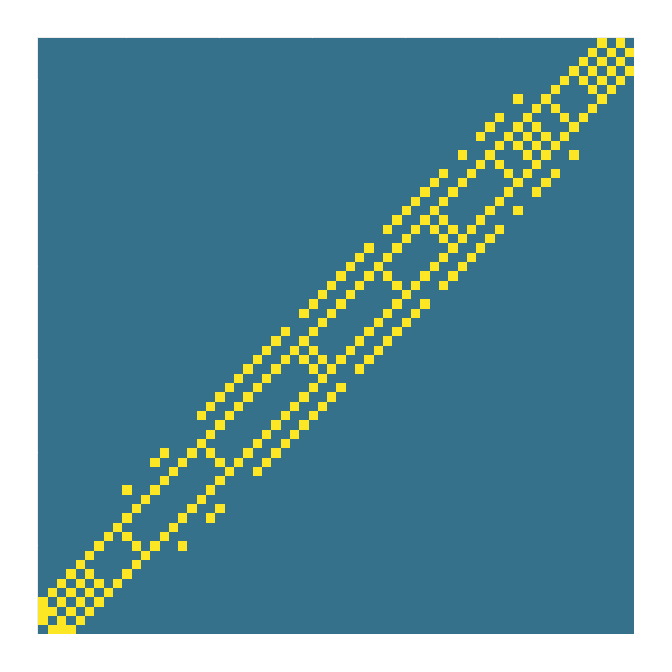}
\caption{Control subject}
\end{subfigure}
\caption{Support (off-diagonals) of SyGlasso-estimated temporal Sylvester factor $\hat{\mat{\Psi}}_{time}$ of the precision matrix for (a) the alcoholic subject and (b) the control subject. Both subjects exhibit banded conditional dependency structures over time.}
\label{fig:alcoholism_time_result}
\end{figure}
While previous analyses on this EEG data using graphical models only focused on the precision matrix of the electrodes, here we exhibit in Figure \ref{fig:alcoholism_time_result} the second precision matrix that encodes temporal dependency. Overall both subjects exhibit banded dependency structures over time, since adjacent timepoints are conditionally dependent. However, note that the conditional dependency structure of the timepoints for the alcoholic subject appears to be more chaotic.

\section{Discussion}
This paper proposed Sylvester-structured graphical model and an inference algorithm, the SyGlasso, that can be applied to tensor-valued data. The current tools available for researchers are limited to Kronecker product and Kronecker sum models on either the covariance or the precision matrix. Our model is motivated by a generative stochastic representation based on the Sylvester equation. We showed that the resulting precision matrix corresponds to the squared Kronecker sum of the precision matrices $\mat{\Psi}_k$ along each mode. The individual components $\mat{\Psi}_k$'s are estimated by the nodewise regression based approach. 

There are several promising future directions to take the proposed SyGlasso. First is to relax the assumption that the diagonals of the factors are fixed - an assumption that is standard among the Kronecker structured models for theoretical analysis. Practically, SyGlasso is able to recover the off-diagonals of the individual $\bm\Psi_k$ and the diagonal of $\bm\Omega$, which only requires to estimating $\bigoplus_{k=1}^K \text{diag}(\bm{\Psi}_k)$ instead of all diagonal entries $\text{diag}(\bm{\Psi_k})$ for all $k$. However, we believe that analyzing the sparsity pattern of the squared Kronecker sum matrix would help us estimate the diagonal entries of the individual components $\mat{\Psi}_k$'s. In addition, it would be worthwhile to study optimization procedures that perform matrix-wise estimation of $\bm\Psi_k$ (yielding simultaneous estimates for both off-diagonal and diagonal entries) and compare its empirical and theoretical properties with the approach proposed in this paper.

Secondly, in terms of the statistical properties, our theoretical results guarantee sparsistency of the individual graphs with a slower convergence rate than that is proposed in \citet{greenewald2017tensor}, while empirical evidence suggests that a faster rate can be achieved. Improvement of this statistical convergence rate analysis will be worthwhile. Also, our results do not guarantee the statistical convergence of individual $\mat\Psi_k$'s nor $\bm\Omega$ with respect to the operator norm. Similar to the solution proposed in \citet{zhou2011high}, we plan to adopt a two-step procedure using SyGlasso for variable selection followed by refitting the precision matrix $\mat{\Omega}$ using maximum likelihood estimation with edge constraint.

Lastly, an exciting future direction is to investigate the utility of SyGlasso as a tool to perform system identification for systems that can be modeled by Sylvester approximations to differential equations. Such dynamical systems include, for example, physical processes governed by separable elliptic PDEs as described in \citet{grasedyck2004existence}, \citet{kressner2010krylov}. It is likely that myriad physical processes can be well-modeled using a Sylvester equation, leading to sparse precision matrices (e.g., a space-time process satisfying the Poisson equation).

\subsubsection*{Acknowledgments}
The authors acknowledge US agency support by grants ARO W911NF-15-1-0479 and DOE DE-NA0003921.

\vfill
\clearpage

\bibliographystyle{apalike}
\bibliography{syglasso_aistats}

\vfill
\clearpage

\appendix
\onecolumn

\section*{Appendix}
\begin{itemize}
    \item[] \textbf{A} provides the detailed derivation of the updates for Algorithm \ref{alg:nodewise_tensor_lasso}. 
    \item[] \textbf{B} provides the proofs of theorems stated in Section \ref{thm}. 
    \item[] \textbf{C} provides details on the simulated data in Section \ref{numeric}. 
\end{itemize}

\section{Derivation of the Nodewise Tensor Lasso Estimator}
\label{sec:sylvester_teralasso_derivation}

\subsection{Off-Diagonal updates}
\label{subsec:derivation_offdiag}
For $1 \leq i_k < j_k \leq m_k$, $T_{i_kj_k}(\mat{\Psi}_k^{\text{off}})$ can be computed in closed form:
\begin{equation}\label{eqn:update_offdiag}
    (T_{i_kj_k}(\mat{\Psi}_k))_{i_kj_k}^{\text{off}} = 
    \frac{S_{\frac{\lambda_k}{N}}\Big(F_{\tensor{X},\{\mat{\Psi}_k\}_{k=1}^K}\Big)}
    {
    (\frac{1}{N}\tensor{X}_{(k)}\tensor{X}_{(k)}^T)_{i_ki_k} + (\frac{1}{N}\tensor{X}_{(k)}\tensor{X}_{(k)}^T)_{j_kj_k}
    },
\end{equation}
where
\begin{align*}
        F_{\tensor{X},\{\mat{\Psi}_k\}_{k=1}^K} = - \frac{1}{N} &
        \Bigg(\Big((\tensor{W}_{(k)} \circ \tensor{X}_{(k)}) \tensor{X}_{(k)}^T \Big)_{i_kj_k} + \Big((\tensor{W}_{(k)} \circ \tensor{X}_{(k)}) \tensor{X}_{(k)}^T \Big)_{j_ki_k} \\
        & + \Big(\tensor{X}_{(k)}(\tensor{X} \times_k \mat{\Psi}_k^{\text{off},i_kj_k})^T_{(k)}\Big)_{j_ki_k} + \Big(\tensor{X}_{(k)}(\tensor{X} \times_k \mat{\Psi}_k^{\text{off},i_kj_k})^T_{(k)}\Big)_{i_kj_k} \\
        & \quad + \sum_{l \neq k} \Big(\tensor{X}_{(k)}(\tensor{X} \times_l \mat{\Psi}_l^{\text{off}})^T_{(k)}\Big)_{i_kj_k} + \sum_{l \neq k} \Big(\tensor{X}_{(k)}(\tensor{X} \times_l \mat{\Psi}_l^{\text{off}})^T_{(k)}\Big)_{j_ki_k}
        \Bigg).
\end{align*} 
Here the $\circ$ operator denotes the Hadamard product between matrices; $\mat{\Psi}_k^{\text{off},i_kj_k}$ is $\mat{\Psi}_k^{\text{off}}$ with the $(i_k,j_k)$ entry being zero; and $S_{\lambda}(x):=\text{sign}(x)(|x|-\lambda)_{+}$ is the soft-thresholding operator. 

\subsection{Diagonal updates}
\label{subsec:derivation_diag}
For $\tensor{W}$,
\begin{equation}
\label{eqn:update_diag}
(T(\tensor{W}))_{i_{[1:K]}}
    = \frac{-\Big(\tensor{X}_{(N)}^T\tensor{Y}_{(N)}\Big)_{i_{[1:K]}}+\sqrt{\Big(\tensor{X}_{(N)}^T\tensor{Y}_{(N)}\Big)_{i_{[1:K]}}^2+4\Big(\tensor{X}_{(N)}\tensor{X}_{(N)}^T\Big)_{i_{[1:K]}}}}
    {2 \Big(\tensor{X}_{(N)}\tensor{X}_{(N)}^T\Big)_{i_{[1:K]}}}.
\end{equation}
Here we define $\tensor{Y}:=\sum_{k=1}^K \Big(\tensor{X} \times_k \mat{\Psi}_k^{\text{off}} \Big)$. Equations \eqref{eqn:update_offdiag} and \eqref{eqn:update_diag} give necessary ingredients for designing a coordinate descent approach to minimizing the objective function in \eqref{eqn:objective}. The optimization procedure is summarized in Algorithm \ref{alg:nodewise_tensor_lasso}.
\subsection{Derivation of updates}
Note that for $1 \leq i_k < j_k \leq m_k$, $1 \leq k \leq K$,
\begin{align*}
    & Q_N(\{\mat{\Psi}_k\}_{k=1}^K) \\ 
    & = (N/2)\Big(\sum_{i_{[1:k-1,k+1:K]}} ({\tensor{X}_{i_{[1:K]}}^{i_k}}^2 +{\tensor{X}_{i_{[1:K]}}^{j_k}}^2)\Big)\Big((\mat{\Psi}_k)_{i_kj_k}\Big)^2 \\ 
    & + N F_{\tensor{X},\{\mat{\Psi}\}_{k=1}^K} (\mat{\Psi}_k)_{i_kj_k} + \lambda_k|(\mat{\Psi}_k)_{i_kj_k}| \\ 
    & + \text{terms independent of $(\mat{\Psi}_k)_{i_kj_k}$},
\end{align*} where

\begin{equation*}
  \begin{aligned}
        F_{\tensor{X},\{\mat{\Psi}\}_{k=1}^K} = - \sum_{i_{[1:k-1,k+1:K]}} &
        \Big( \tensor{W}_{i_{[1:K]}}^{i_k} \tensor{X}_{i_{[1:K]}}^{i_k}\tensor{X}_{i_{[1:K]}}^{j_k} + \tensor{W}_{i_{[1:K]}}^{j_k} \tensor{X}_{i_{[1:K]}}^{j_k}\tensor{X}_{i_{[1:K]}}^{i_k}\\
        & \quad + (\mat{\Psi}_k)_{i_k,\bslash \{i_k,j_k\}}^T \tensor{X}_{i_{[1:K]}}^{\bslash \{i_k,j_k\}} \tensor{X}_{i_{[1:K]}}^{j_k}  \\
        & \quad + (\mat{\Psi}_k)_{j_k,\bslash \{i_k,j_k\}}^T \tensor{X}_{i_{[1:K]}}^{\bslash \{i_k,j_k\}}
        \tensor{X}_{i_{[1:K]}}^{i_k}\\
        & \quad + \sum_{l \in [1:k-1,k+1:K]}
        (\mat{\Psi}_l)_{i_l,\bslash i_l}^T \tensor{X}_{i_{[1:K]}}^{i_k,\bslash i_l} \tensor{X}_{i_{[1:K]}}^{j_k} \\
        & \quad + \sum_{l \in [1:k-1,k+1:K]}
        (\mat{\Psi}_l)_{i_l,\bslash i_l}^T \tensor{X}_{i_{[1:K]}}^{j_k,\bslash i_l} \tensor{X}_{i_{[1:K]}}^{i_k}
         \Big).
    \end{aligned}
\end{equation*} Here $\tensor{X}_{i_{[1:K]}}^{i_k}$ denotes the element of $\tensor{X}$ indexed by $i_{[1:K]}$ except that the $k$th index is replaced by $i_k$ and $\tensor{X}_{i_{[1:K]}}^{i_k,j_l}$ denotes the element of $\tensor{X}$ indexed by $i_{[1:K]}$ except that the $k,l$th indices are replaced by $i_k,j_l$. Note the following equivalence:
\begin{align*}
    & \sum_{i_{[1:k-1,k+1:K]}} \tensor{W}_{i_{[1:K]}}^{i_k} \tensor{X}_{i_{[1:K]}}^{i_k}\tensor{X}_{i_{[1:K]}}^{j_k}=\Big((\tensor{W}_{(k)} \circ \tensor{X}_{(k)}) \tensor{X}_{(k)}^T \Big)_{i_kj_k} \\ & \sum_{i_{[1:k-1,k+1:K]}}\tensor{X}_{i_{[1:K]}}^{i_k}\tensor{X}_{i_{[1:K]}}^{j_k} = (\tensor{X}_{(k)}\tensor{X}_{(k)}^T)_{i_kj_k} \\ & \sum_{i_{[1:k-1,k+1:K]}}(\mat{\Psi}_l)_{i_l,.}^T \tensor{X}_{i_{[1:K]}}^{i_k,.} \tensor{X}_{i_{[1:K]}}^{j_k} = \Big(\tensor{X}_{(k)}(\tensor{X}\times_l\mat{\Psi}_l)_{(k)}^T\Big)_{j_ki_k},
\end{align*}
where $\tensor{W}$ is a tensor of the same dimensions of $\tensor{X}$, formed by tensorize values in $\tensor{W}$, and in the case of $N>1$ the last mode of $\tensor{W}$ is the observation mode similarly to $\tensor{X}$ but with exact replicates. Using the tensor notation and standard sub-differential method, Equation \eqref{eqn:update_offdiag} then follows. 

For $\tensor{W}_{i_{[1:K]}}$, using similar tensor operations,
\begin{align*}
    & \frac{\partial}{\partial \tensor{W}_{i_{[1:K]}}} Q_N(\tensor{W},\{\mat{\Psi}_k^{\text{off}}\}_{k=1}^K)  = 0 \\
    & \iff -\frac{1}{\tensor{W}_{i_{[1:K]}}} + \tensor{W}_{i_{[1:K]}}^2 \tensor{X}_{i_{[1:K]}}^2 + \tensor{W}_{i_{[1:K]}}\Big(\tensor{X}_{i_{[1:K]}}\sum_{k=1}^K(\tensor{X} \times_k \mat{\Psi}_k^{\text{off}})_{i_{[1:K]}})\Big) = 0 \\
    & \iff \tensor{W}_{i_{[1:K]}}^2 \Big(\tensor{X}_{(N)}^T\tensor{X}_{(N)}\Big)_{i_{[1:K]}} + \tensor{W}_{i_{[1:K]}} \Big(\tensor{X}_{(N)}^T\sum_{k=1}^K(\tensor{X} \times_k \mat{\Psi}_k^{\text{off}})\Big)_{i_{[1:K]}} - 1 = 0
\end{align*} which is a quadratic equation in $\tensor{W}_{i_{[1:K]}}$ and since $\tensor{W}_{i_{[1:K]}}>0$, so the positive root has been retained as the solution. Note that the estimation for one entry of $\tensor{W}$ is independent of the other entries. So during the estimation process we update all the entries at once by noting that $\diag\Big(\tensor{X}_{(N)}^T\tensor{X}_{(N)}\Big)=\Big(\Big(\tensor{X}_{(N)}^T\tensor{X}_{(N)}\Big)_{i_{[1:K]}}, \forall i_{[1:K]} \Big)$.

\section{Proofs of Main Theorems}\label{proofs}
We first list some properties of the loss function.

\begin{lemma}
The following is true for the loss function:
\begin{enumerate}[label=(\roman*)]
    \item There exist constants $0 < \Lambda_{\min}^L \leq \Lambda_{\max}^L < \infty$ such that for $\mathcal{S}_{k}:=\{(i_k,j_k):1 \leq i_k < j_k \leq m_k\},k=1,\dots,K$,
    \begin{equation*}
        \Lambda_{\min}^L \leq \lambda_{\min}(\bar{L}^{\prime\prime}_{\mathcal{S}_{k},\mathcal{S}_{k}}(\bar{\bm{\beta}})) \leq \lambda_{\max}(\bar{L}^{\prime\prime}_{\mathcal{S}_{k},\mathcal{S}_{k}}(\bar{\bm{\beta}})) \leq \Lambda_{\max}^L
    \end{equation*}
    \item There exists a constant $K(\bar{\bm{\beta}})<\infty$ such that for all $1 \leq i_k < j_k \leq m_k$, $\bar{L}^{\prime\prime}_{i_kj_k,i_kj_k}(\bar{\bm{\beta}}) \leq K(\bar{\bm{\beta}})$
    \item There exist constant $M_1(\bar{\bm{\beta}}), M_2(\bar{\bm{\beta}})<\infty$, such that for any $1 \leq i_k < j_k \leq m_k$
    \begin{equation*}
        \Var_{\bar{\tensor{W}},\bar{\bm{\beta}}}(L^{\prime}_{i_kj_k}(\bar{\tensor{W}},\bar{\bm{\beta}},\tensor{X})) \leq M_1(\bar{\bm{\beta}}), \ \Var_{\bar{\tensor{W}},\bar{\bm{\beta}}}(L^{\prime\prime}_{i_kj_k,i_kj_k}(\bar{\tensor{W}},\bar{\bm{\beta}},\tensor{X})) \leq M_2(\bar{\bm{\beta}})
    \end{equation*}
    \item There exists a constant $0 < g(\bar{\bm{\beta}}) <\infty$, such that for all $(i,j) \in \mathcal{A}_{k}$
    \begin{equation*}
        \bar{L}^{\prime\prime}_{ij,ij}(\bar{\tensor{W}},\bar{\bm{\beta}}) - \bar{L}^{\prime\prime}_{ij,\mathcal{A}_{k}^{ij}}(\bar{\tensor{W}},\bar{\bm{\beta}})[\bar{L}^{\prime\prime}_{\mathcal{A}_{k}^{ij},\mathcal{A}_{k}^{ij}}(\bar{\tensor{W}},\bar{\bm{\beta}})]^{-1}\bar{L}^{\prime\prime}_{\mathcal{A}_{k}^{ij},ij}(\bar{\tensor{W}},\bar{\bm{\beta}}) \geq g(\bar{\bm{\beta}}),
    \end{equation*} where $\mathcal{A}_{k}^{ij}:=\mathcal{A}_{k}/\{(i,j)\}$.
    \item There exists a constant $M(\bar{\bm{\beta}})<\infty$, such that for any $(i,j) \in \mathcal{A}_{k}^c$
    \begin{equation*}
        \|\bar{L}^{\prime\prime}_{ij,\mathcal{A}_{k}}(\bar{\tensor{W}},\bar{\bm{\beta}})[\bar{L}^{\prime\prime}_{\mathcal{A}_{k},\mathcal{A}_{k}}(\bar{\tensor{W}},\bar{\bm{\beta}})]^{-1}\|_2 \leq M(\bar{\bm{\beta}}).
    \end{equation*}
\end{enumerate}
\end{lemma}

\begin{proof}[proof of Lemma B.1.]
We prove $(i)$. $(ii-v)$ are then direct consequences, and the proofs follow from the proofs of B1.1-B1.4 in \citet{peng2009partial}, with the modifications being that the indexing is now with respect to each $k$ for $1 \leq k \leq K$.

Consider the loss function in matrix form as in \eqref{eqn:objective_matrix}. Then $\bar{L}^{\prime\prime}_{\mathcal{S}_{k},\mathcal{S}_{k}}(\bar{\bm{\beta}})$ is equivalent to $\frac{\partial^2}{\partial\mat{\Psi}_k^{\text{off}} \partial\mat{\Psi}_k^{\text{off}} } L(\tensor{W},\{\mat{\Psi}_k^{\text{off}}\}_{k=1}^K)$, which is
\begin{align*}
& \frac{\partial^2}{\partial\mat{\Psi}_k^{\text{off}} \partial\mat{\Psi}_k^{\text{off}}} \Bigg(\tr(\mat{\Psi}_k^T \mat{S} \mat{\Psi}_k) + \text{first order terms in $\mat{\Psi}_k$} + \text{terms independent of $\mat{\Psi}_k$} \Bigg) \\
& = \frac{\partial^2}{\partial\mat{\Psi}_k^{\text{off}} \partial\mat{\Psi}_k^{\text{off}}} \Bigg(\tr((\mat{\Psi}_k^{\text{off}} + \text{diag}(\mat{\Psi}_k))^T \mat{S} (\mat{\Psi}_k^{\text{off}} + \text{diag}(\mat{\Psi}_k))) + \text{first order terms in $\mat{\Psi}_k^{\text{off}}$} \\
& \qquad \qquad \qquad \quad + \text{terms independent of $\mat{\Psi}_k^{\text{off}}$} \Bigg) \\ 
& = \frac{\partial^2}{\partial\mat{\Psi}_k^{\text{off}} \partial\mat{\Psi}_k^{\text{off}}} \Bigg(\tr((\mat{\Psi}_k^{\text{off}})^T \mat{S} \mat{\Psi}_k^{\text{off}}) + \text{first order terms in $\mat{\Psi}_k^{\text{off}}$} +  \text{terms independent of $\mat{\Psi}_k^{\text{off}}$} \Bigg) \\
& = \mat{S} = \frac{1}{N}\vecto(\tensor{X})^T\vecto(\tensor{X}).
\end{align*}
Thus $\bar{L}^{\prime\prime}_{\mathcal{S}_{k},\mathcal{S}_{k}}(\bm{\beta})=E_{\tensor{W},\bm{\beta}}(\mat{S})$. Then for any non-zero $\mat{a} \in \mathbb{R}^p$, we have
\begin{equation*}
    \mat{a}^T \bar{L}^{\prime\prime}_{\mathcal{S}_{k},\mathcal{S}_{k}}(\bar{\bm{\beta}}) \mat{a} = \mat{a}^T \mat{\bar{\Sigma}} \mat{a} \geq \|\mat{a}\|_2^2 \lambda_{\min}(\bar{\mat{\Sigma}}).
\end{equation*}Similarly, $\mat{a}^T \bar{L}^{\prime\prime}_{\mathcal{S}_{k},\mathcal{S}_{k}}(\bar{\bm{\beta}}) \mat{a} \leq \|\mat{a}\|_2^2 \lambda_{\max}(\bar{\mat{\Sigma}})$. By (A2), $\bar{\mat{\Sigma}}$ has bounded eigenvalues, thus the lemma is proved.

\end{proof}

\begin{lemma}
Suppose conditions (A1-A2) hold, then for any $\eta>0$, there exist constant $c_{0,\eta},c_{1,\eta},c_{2,\eta},c_{3,\eta}$, such that for any $u \in \mathbb{R}^{q_{k}}$ the following events hold with probability at least $1-O(\exp(-\eta \log p))$ for sufficiently large $N$:
\begin{enumerate}[label=(\roman*)]
    \item $\|L_{N,\mathcal{A}_{k}}^{\prime}(\bar{\tensor{W}},\bar{\bm{\beta}},\tensor{X})\|_2 \leq c_{0,\eta}\sqrt{q_{k}\frac{\log p}{N}}$
    \item $|u^T L_{N,\mathcal{A}_{k}}^{\prime}(\bar{\tensor{W}},\bar{\bm{\beta}},\tensor{X})| \leq c_{1,\eta}\|u\|_2\sqrt{q_{k}\frac{\log p}{N}}$
    \item $|u^T L_{N,\mathcal{A}_{k}\mathcal{A}_{k}}^{\prime\prime}(\bar{\tensor{W}},\bar{\bm{\beta}},\tensor{X})u - u^T \bar{L}^{\prime\prime}_{\mathcal{A}_{k}\mathcal{A}_{k}}(\bar{\bm{\beta}})u| \leq c_{2,\eta}\|u\|_2^2q_{k}\sqrt{\frac{\log p}{N}}$
    \item $|L_{N,\mathcal{A}_{k}\mathcal{A}_{k}}^{\prime\prime}(\bar{\tensor{W}},\bar{\bm{\beta}},\tensor{X})u - \bar{L}^{\prime\prime}_{\mathcal{A}_{k}\mathcal{A}_{k}}(\bar{\bm{\beta}})u| \leq c_{3,\eta}\|u\|_2^2q_{k}\sqrt{\frac{\log p}{N}}$
\end{enumerate}
\end{lemma}

\begin{proof}[proof of Lemma B.2.]
$(i)$ By Cauchy-Schwartz inequality,
\begin{equation*}
    \|L^{\prime}_{N,\mathcal{A}_{k}}(\bar{\tensor{W}},\bar{\bm{\beta}},\tensor{X})\|_2 \leq \sqrt{q_{k}} \max_{i \in \mathcal{A}_{k}} |L^{\prime}_{N,i}(\bar{\tensor{W}},\bar{\bm{\beta}},\tensor{X})|.
\end{equation*} Then note that 
\begin{align*}
    & L^{\prime}_{N,i}(\tensor{W},\bm{\beta},\tensor{X}) \\
    & = \sum_{i_{[1:k-1,k+1:K]}} (e_{i_{[1:k-1]},p,i_{[k+1:K]}}(\tensor{W},\bm{\beta})\tensor{X}_{i_{[1:k-1]},q,i_{[k+1:K]}} + e_{i_{[1:k-1]},q,i_{[k+1:K]}}(\tensor{W},\bm{\beta})\tensor{X}_{i_{[1:k-1]},p,i_{[k+1:K]}}),
\end{align*}where $e_{i_{[1:k-1]},p,i_{[k+1:K]}}\tensor{X}_{i_{[1:k-1]},q,i_{[k+1:K]}}(\tensor{W},\bm{\beta})$ is defined by
\begin{equation*}
    w_{i_{[1:k-1]},p,i_{[k+1:K]}}\tensor{X}_{i_{[1:k-1]},p,i_{[k+1:K]}} + \sum_{j_k \neq p} (\mat{\Psi}_k)_{p,j_k}\tensor{X}_{i_{[1:k-1]},j_k,i_{[k+1:K]}}
    + \sum_{l \neq k} \sum_{j_l \neq i_l} (\mat{\Psi}_l)_{i_l,j_l}\tensor{X}_{i_{[1:k-1]},p,i_{[k+1:K]}}.
\end{equation*}
Then evaluated at the true parameter values $(\bar{\tensor{W}},\bar{\bm{\beta}})$, we have $e_{i_{[1:k-1]},p,i_{[k+1:K]}}(\bar{\tensor{W}},\bar{\bm{\beta}})$ uncorrelated with $\tensor{X}_{i_{[1:k-1]},\bslash p,i_{[k+1:K]}}$ and $E_{(\bar{\tensor{W}},\bar{\bm{\beta}})}(e_{i_{[1:k-1]},p,i_{[k+1:K]}}(\bar{\tensor{W}},\bar{\bm{\beta}}))=0$. Also, since $\tensor{X}$ is subgaussian and $\Var(L^{\prime}_{N,i}(\bar{\tensor{W}},\bar{\bm{\beta}},\tensor{X}))$ is bounded by Lemma C.1. $\forall i$, $L^{\prime}_{N,i}(\bar{\tensor{W}},\bar{\bm{\beta}},\tensor{X})$ has subexponential tails. Thus, by Bernstein inequality,
\begin{align*}
    & P(\|L_{N,\mathcal{A}_{k}}^{\prime}(\bar{\tensor{W}},\bar{\bm{\beta}},\tensor{X})\|_2 \leq c_{0,\eta}\sqrt{q_{k}\frac{\log p}{N}}) \\
    & \geq P(\sqrt{q_{k}} \max_{i \in \mathcal{A}_{k}} |L^{\prime}_{N,i}(\bar{\tensor{W}},\bar{\bm{\beta}},\tensor{X})| \leq c_{0,\eta}\sqrt{q_{k}\frac{\log p}{N}}) \geq 1 - O(\exp(-\eta \log p)).
\end{align*}

$(iii)$ By Cauchy-Schwartz,
\begin{align*}
    & |u^T L_{N,\mathcal{A}_{k}\mathcal{A}_{k}}^{\prime\prime}(\bar{\tensor{W}},\bar{\bm{\beta}},\tensor{X})u - u^T \bar{L}^{\prime\prime}_{\mathcal{A}_{k}\mathcal{A}_{k}}(\bar{\bm{\beta}})u| \\
    & \leq \|u\|_2 \|u^T L_{N,\mathcal{A}_{k}\mathcal{A}_{k}}^{\prime\prime}(\bar{\tensor{W}},\bar{\bm{\beta}},\tensor{X}) - u^T \bar{L}^{\prime\prime}_{\mathcal{A}_{k}\mathcal{A}_{k}}(\bar{\bm{\beta}})\|_2 \\
    & \leq \|u\|_2 \sqrt{q_{k}} \max_i |u^T L_{N,\mathcal{A}_{k},i}^{\prime\prime}(\bar{\tensor{W}},\bar{\bm{\beta}},\tensor{X}) - u^T \bar{L}^{\prime\prime}_{\mathcal{A}_{k},i}(\bar{\bm{\beta}})| \\
    & = \|u\|_2 \sqrt{q_{k}} |u^T L_{N,\mathcal{A}_{k},i_{\max}}^{\prime\prime}(\bar{\tensor{W}},\bar{\bm{\beta}},\tensor{X}) - u^T \bar{L}^{\prime\prime}_{\mathcal{A}_{k},i_{\max}}(\bar{\bm{\beta}})| \\
    & = \|u\|_2 \sqrt{q_{k}} |\sum_{j=1}^{q_{k}} (u_j L_{N,j,i_{\max}}^{\prime\prime}(\bar{\tensor{W}},\bar{\bm{\beta}},\tensor{X}) - u_j \bar{L}^{\prime\prime}_{j,i_{\max}}(\bar{\bm{\beta}}))| \\
    & \leq \|u\|_2 q_{k} |u_{j_{\max}}|| L_{N,j_{\max},i_{\max}}^{\prime\prime}(\bar{\tensor{W}},\bar{\bm{\beta}},\tensor{X}) -  \bar{L}^{\prime\prime}_{j_{\max},i_{\max}}(\bar{\bm{\beta}}))| \\ 
    & \leq \|u\|_2^2 q_{k} | L_{N,j_{\max},i_{\max}}^{\prime\prime}(\bar{\tensor{W}},\bar{\bm{\beta}},\tensor{X}) -  \bar{L}^{\prime\prime}_{j_{\max},i_{\max}}(\bar{\bm{\beta}}))|.
\end{align*}Then by Bernstein inequality,
\begin{align*}
    & P(|u^T L_{N,\mathcal{A}_{k}\mathcal{A}_{k}}^{\prime\prime}(\bar{\tensor{W}},\bar{\bm{\beta}},\tensor{X})u - u^T \bar{L}^{\prime\prime}_{\mathcal{A}_{k}\mathcal{A}_{k}}(\bar{\bm{\beta}})u| \leq c_{2,\eta}\|u\|_2^2q_{k}\sqrt{\frac{\log p}{N}}) \\
    & \geq P(\|u\|_2^2 q_{k} | L_{N,j_{\max},i_{\max}}^{\prime\prime}(\bar{\tensor{W}},\bar{\bm{\beta}},\tensor{X}) -  \bar{L}^{\prime\prime}_{j_{\max},i_{\max}}(\bar{\bm{\beta}}))| \leq c_{2,\eta}\|u\|_2^2q_{k}\sqrt{\frac{\log p}{N}}) \\
    & \geq 1 - O(\exp(-\eta \log p)).
\end{align*}

$(ii)$ and $(iv)$ can be proved using similar arguments.
\end{proof}

Lemma C.3. and C.4. are used later to prove Theorem 1.

\begin{lemma}
Assuming conditions of Theorem 1. Then there exists a constant $C_1(\bar{\bm{\beta}})>0$ such that for any $\eta>0$, there exists a global minimizer of the restricted problem \eqref{eqn:restricted_problem} within the disc:
\begin{equation*}
    \{\bm{\beta}: \|\bm{\beta}-\bar{\bm{\beta}}\|_2 \leq C_1(\bar{\bm{\beta}}) \sqrt{K} \max_k\sqrt{q_{k}}\lambda_{N,k} \}
\end{equation*} with probability at least $1 - O(\exp(-\eta \log p))$ for sufficiently large $N$.
\end{lemma}

\begin{proof}[proof of Lemma B.3.]
Let $\alpha_N = \max_{k}\sqrt{q_{k}}\lambda_{N,k}$. Further for $1 \leq k \leq K$ let $C_k>0$ and $u^k \in \mathbb{R}^{m_k(m_k-1)/2}$ such that $u_{\mathcal{A}_{k}^c}^k=0$, $\|u^k\|_2=C_k$, and  $u=(u_1,\dots,u_K)$ with $\sqrt{K}\min_kC_k \leq \|u\|_2 \leq \sqrt{K}\max_kC_k$.

Then by Cauchy-Schwartz and triangle inequality, we have
\begin{equation*}
    \|\bar{\bm{\beta}}^k + \alpha_N u^k - \alpha_N u^k\|_1
    \leq \|\bar{\bm{\beta}}^k + \alpha_N u^k\|_1 + \alpha_N\|u^k\|_1,
\end{equation*} and
\begin{equation*}
    \|\bar{\bm{\beta}}^k\|_1 - \|\bar{\bm{\beta}}^k + \alpha_N u^k\|_1 \leq \alpha_N\|u^k\|_1 \leq \alpha_N \sqrt{q_{k}} \|u^k\|_2 = C_k \alpha_N \sqrt{q_{k}}.
\end{equation*} Thus,
\begin{align*}
    & Q_N(\bar{\bm{\beta}} + \alpha_N u, \tensor{X},\{\lambda_{N,k}\}_{k=1}^K) - Q_N(\bar{\bm{\beta}}, \tensor{X},\{\lambda_{N,k}\}_{k=1}^K) \\
    & = L_N(\bar{\bm{\beta}} + \alpha_N u, \tensor{X}) - L_N(\bar{\bm{\beta}} , \tensor{X}) - \sum_{k=1}^K \lambda_{N,k} \big(\|\bar{\bm{\beta}}^k\|_1 - \|\bar{\bm{\beta}}^k + \alpha_N u^k\|_1 \big) \\
    & \geq L_N(\bar{\bm{\beta}} + \alpha_N u, \tensor{X}) - L_N(\bar{\bm{\beta}} , \tensor{X}) - \sum_{k=1}^K \lambda_{N,k} C_k \alpha_N\sqrt{q_{k}} \\
    & \geq L_N(\bar{\bm{\beta}} + \alpha_N u, \tensor{X}) - L_N(\bar{\bm{\beta}} , \tensor{X}) - \alpha_N K \max_k C_k\sqrt{q_{k}}\lambda_{N,k} \\
    & \geq L_N(\bar{\bm{\beta}} + \alpha_N u, \tensor{X}) - L_N(\bar{\bm{\beta}} , \tensor{X}) - K \alpha_N^2\max_kC_k.
\end{align*}Next, 
\begin{align*}
    & L_N(\bar{\bm{\beta}} + \alpha_N u, \tensor{X}) - L_N(\bar{\bm{\beta}},\tensor{X}) = \alpha_N u^T_{\mathcal{A}} L_{N,\mathcal{A}}^{\prime}(\bar{\bm{\beta}},\tensor{X}) + \frac{1}{2}\alpha_N^2 u^T_{\mathcal{A}}L_{N,\mathcal{A}\mathcal{A}}^{\prime\prime}(\bar{\bm{\beta}},\tensor{X})u_{\mathcal{A}} \\
    & = \alpha_N \sum_{k=1}^K (u^k_{\mathcal{A}_{k}})^T L_{N,\mathcal{A}_{k}}^{\prime}(\bar{\bm{\beta}},\tensor{X}) + \frac{1}{2}\alpha_N^2 \sum_{k=1}^K (u^k_{\mathcal{A}_{k}})^T L_{N,\mathcal{A}_{k}\mathcal{A}_{k}}^{\prime\prime}(\bar{\bm{\beta}},\tensor{X})u^k_{\mathcal{A}_{k}} \\
    & = \alpha_N \sum_{k=1}^K (u^k_{\mathcal{A}_{k}})^T L_{N,\mathcal{A}_{k}}^{\prime}(\bar{\bm{\beta}},\tensor{X}) + \frac{1}{2}\alpha_N^2 \sum_{k=1}^K (u^k_{\mathcal{A}_{k}})^T (L_{N,\mathcal{A}_{k}\mathcal{A}_{k}}^{\prime\prime}(\bar{\bm{\beta}},\tensor{X}) - \bar{L}_{N,\mathcal{A}_{k}\mathcal{A}_{k}}^{\prime\prime}(\bar{\bm{\beta}},\tensor{X}))u^k_{\mathcal{A}_{k}} \\
    & + \frac{1}{2}\alpha_N^2 \sum_{k=1}^K (u^k_{\mathcal{A}_{k}})^T \bar{L}_{N,\mathcal{A}_{k}\mathcal{A}_{k}}^{\prime\prime}(\bar{\bm{\beta}},\tensor{X}) u^k_{\mathcal{A}_{k}} \\
    & \geq \frac{1}{2}\alpha_N^2 \sum_{k=1}^K (u^k_{\mathcal{A}_{k}})^T \bar{L}_{N,\mathcal{A}_{k}\mathcal{A}_{k}}^{\prime\prime}(\bar{\bm{\beta}},\tensor{X}) u^k_{\mathcal{A}_{k}} - \alpha_N K(\max_k c_{1,\eta}\|u^k_{\mathcal{A}_{k}}\|_2\sqrt{q_{k}\frac{\log p}{N}}) \\
    & - \frac{1}{2} \alpha_N^2 K(\max_k c_{2,\eta}\|u^k_{\mathcal{A}_{k}}\|_2^2q_{k}\sqrt{\frac{\log p}{N}}).
\end{align*} Here the first equality is due to the second order expansion of the loss function and the inequality is due to Lemma B.2. For sufficiently large $N$, by assumption that $\lambda_{N,k}\sqrt{N/\log p} \rightarrow \infty$ if $m_k \rightarrow \infty$ and $\sqrt{\log p / N}=o(1)$, the second term in the last line above is $o(\alpha_N \sqrt{q_{k}} \lambda_{N,k}) = o(\alpha_{N}^2)$; the last term is $o(\alpha_N^2)$. Therefore, for sufficiently large $N$
\begin{align*}
    Q_N(\bar{\bm{\beta}} + \alpha_N u, \tensor{X},\{\lambda_{N,k}\}_{k=1}^K) - Q_N(\bar{\bm{\beta}}, \tensor{X},\{\lambda_{N,k}\}_{k=1}^K) & \geq \frac{1}{2}\alpha_N^2 \sum_{k=1}^K (u^k_{\mathcal{A}_{k}})^T \bar{L}_{N,\mathcal{A}_{k}\mathcal{A}_{k}}^{\prime\prime}(\bar{\bm{\beta}},\tensor{X}) u^k_{\mathcal{A}_{k}} \\
    & - K \alpha_N^2\max_kC_k \\
    & \geq  \frac{1}{2}\alpha_N^2 K \min_k \big((u^k_{\mathcal{A}_{k}})^T \bar{L}_{N,\mathcal{A}_{k}\mathcal{A}_{k}}^{\prime\prime}(\bar{\bm{\beta}},\tensor{X}) u^k_{\mathcal{A}_{k}}\big) \\
    & - K \alpha_N^2\max_kC_k,
\end{align*} with probability at least $1-O(N^{-\eta})$. By Lemma B.1., for each $k$, $(u^k_{\mathcal{A}_{k}})^T \bar{L}_{N,\mathcal{A}_{k}\mathcal{A}_{k}}^{\prime\prime}(\bar{\bm{\beta}},\tensor{X}) u^k_{\mathcal{A}_{k}} \geq \Lambda_{\min}^L \|u^k_{\mathcal{A}_{k}}\|_2^2=\Lambda_{\min}^L(C_k)^2$. So, if we choose $\min_k C_k$ and $\max_k C_k$ such that the upper bound is minimized, then for $N$ sufficiently large, the following holds 
\begin{equation*}
    \inf_{u:u_{(\mathcal{A}_{k})^c} = 0,\|u^k\|_2=C_k,k=1,\dots,K} Q_N(\bar{\bm{\beta}} + \alpha_N u, \tensor{X},\{\lambda_{N,k}\}_{k=1}^K) > Q_N(\bar{\bm{\beta}}, \tensor{X},\{\lambda_{N,k}\}_{k=1}^K),
\end{equation*} with probability at least $1-O(\exp(-\eta \log p))$, which means any solution to the problem defined in \eqref{eqn:restricted_problem} is within the disc $\{\bm{\beta}: \|\bm{\beta}-\bar{\bm{\beta}}\|_2 \leq \alpha_N \|u\|_2 \leq \alpha_N \sqrt{K} \max_kC_k\}$ with probability at least $1-O(\exp(-\eta \log p))$.

\end{proof}

\begin{lemma}
Assuming conditions of Theorems 1. Then there exists a constant $C_2(\bar{\bm{\beta}})>0$, such that for any $\eta>0$, for sufficiently large $N$, the following event holds with probability at least $1-O(\exp(-\eta \log p))$: if for any $\bm{\beta} \in S=\{\bm{\beta}: \|\bm{\beta}-\bar{\bm{\beta}}\|_2 \geq C_2(\bar{\bm{\beta}})\sqrt{K}\max_k\sqrt{q_{k}}\lambda_{N,k},\bm{\beta}_{\mathcal{A}_{N}^c}=0\}$, then $\|L^{\prime}_{N,\mathcal{A}_{N}}(\bar{\tensor{W}},\bar{\bm{\beta}},\tensor{X})\|_2 > \sqrt{K}\max_k\sqrt{q_{k}}\lambda_{N,k}$.
\end{lemma}

\begin{proof}[proof of Lemma B.4.]
Let $\alpha_N = \max_k \sqrt{q_{k}}\lambda_{N,k}$. For $\bm{\beta} \in S$, we have $\bm{\beta}=\bar{\bm{\beta}}+\alpha_N u$, with $u_{(\mathcal{A})^c}$ and $\|u\|_2 \geq C_2(\bar{\bm{\beta}})$. Note that by Taylor expansion of $L^{\prime}_{N,\mathcal{A}}(\bar{\tensor{W}},\bm{\beta},\tensor{X})$ at $\bar{\bm{\beta}}$
\begin{align*}
    L^{\prime}_{N,\mathcal{A}}(\bar{\tensor{W}},\bm{\beta},\tensor{X}) & = L^{\prime}_{N,\mathcal{A}}(\bar{\tensor{W}},\bm{\beta},\tensor{X}) + \alpha_NL^{\prime\prime}_{N,\mathcal{A}\mathcal{A}}(\bar{\tensor{W}},\bm{\beta},\tensor{X})u_{\mathcal{A}} \\
    & = L^{\prime}_{N,\mathcal{A}}(\bar{\tensor{W}},\bm{\beta},\tensor{X}) + \alpha_N\big(L^{\prime\prime}_{N,\mathcal{A}\mathcal{A}}(\bar{\tensor{W}},\bm{\beta},\tensor{X})-\bar{L}^{\prime\prime}_{N,\mathcal{A}\mathcal{A}}(\bar{\bm{\beta}})\big)u_{\mathcal{A}} \\
    & \quad + \alpha_N\bar{L}^{\prime\prime}_{N,\mathcal{A}\mathcal{A}}(\bar{\bm{\beta}})u_{\mathcal{A}}.
\end{align*}
By triangle inequality and similar proof strategies as in Lemma B.3., for sufficiently large $N$ 
\begin{align*}
    \|L^{\prime}_{N,\mathcal{A}}(\bar{\tensor{W}},\bm{\beta},\tensor{X})\|_2 & \geq  \|L^{\prime}_{N,\mathcal{A}}(\bar{\tensor{W}},\bm{\beta},\tensor{X})\|_2 + \alpha_N\|L^{\prime\prime}_{N,\mathcal{A}\mathcal{A}}(\bar{\tensor{W}},\bm{\beta},\tensor{X})u_{\mathcal{A}}-\bar{L}^{\prime\prime}_{N,\mathcal{A}\mathcal{A}}(\bar{\bm{\beta}})u_{\mathcal{A}}\|_2 \\
    & \quad + \alpha_N\|\bar{L}^{\prime\prime}_{N,\mathcal{A}\mathcal{A}}(\bar{\bm{\beta}})u_{\mathcal{A}}\|_2 \\
    & \geq \alpha_N\|\bar{L}^{\prime\prime}_{N,\mathcal{A}\mathcal{A}}(\bar{\bm{\beta}})u_{\mathcal{A}}\|_2 + o(\alpha_N)
\end{align*} with probability at least $1-O(\exp(-\eta \log p))$. By Lemma B.1., $\|\bar{L}^{\prime\prime}_{N,\mathcal{A}\mathcal{A}}(\bar{\bm{\beta}})u_{\mathcal{A}}\|_2 \geq \Lambda_{\min}^L(\bar{\bm{\beta}}) \|u_{\mathcal{A}}\|_2$. Therefore, taking $C_2(\bar{\bm{\beta}})$ to be $1/\Lambda_{\min}^L(\bar{\bm{\beta}}) + \epsilon$ completes the proof.
\end{proof}

\begin{proof}[proof of Theorem 1]
By the Karush-Kuhn-Tucker condition, for any solution $\hat{\bm{\beta}}$ of \eqref{eqn:restricted_problem}, it satisfies $\|L_{N,\mathcal{A}_{k}}^{\prime}(\tensor{W},\hat{\bm{\beta}},\tensor{X})\|_{\infty} \leq \lambda_{N,k}$. Thus,
\begin{align*}
    \|L_{N,\mathcal{A}_{N}}^{\prime}(\tensor{W},\hat{\bm{\beta}},\tensor{X})\|_2 & \leq \sqrt{K}\max_k\|L_{N,\mathcal{A}_{k}}^{\prime}(\tensor{W},\hat{\bm{\beta}},\tensor{X})\|_2 \\
    & \leq \sqrt{K}\max_k\sqrt{q_{k}}\|L_{N,\mathcal{A}_{k}}^{\prime}(\tensor{W},\hat{\bm{\beta}},\tensor{X})\|_{\infty} \\
    & \leq \sqrt{K}\max_k\sqrt{q_{k}}\lambda_{N,k}.
\end{align*}Then by Lemmas B.4., for any $\eta >0$, for $N$ sufficiently large, all solutions of \eqref{eqn:restricted_problem} are inside the disc $\{\bm{\beta}: \|\bm{\beta}-\bar{\bm{\beta}}\|_2 \leq C_2(\bar{\bm{\beta}})\max_k\sqrt{q_{k}}\lambda_{N,k},\bm{\beta}_{\mathcal{A}_{N}^c}=0\}$ with probability at least $1-O(\exp(-\eta \log p))$. If we further assume that $\min_{(i,j) \in \mathcal{A}_{k}}|\bar{\bm{\beta}}_{i,j}| \geq 2C(\bar{\bm{\beta}})\max_{k}\sqrt{q_{k}}\lambda_{N,k}$ for each $k$, then
\begin{align*}
    1-O(\exp(-\eta \log p)) \\
    & \leq P_{\bar{\tensor{W}},\bar{\bm{\beta}}}(\|\hat{\bm{\beta}}^{\mathcal{A}}-\bar{\bm{\beta}}^{\mathcal{A}}\|_2 \leq C_2(\bar{\bm{\beta}})\max_k\sqrt{q_{k}}\lambda_{N,k},\min_{(i,j) \in \mathcal{A}_{k}}|\bar{\bm{\beta}}_{i,j}| \geq 2C(\bar{\bm{\beta}})\max_{k}\sqrt{q_{k}}\lambda_{N,k},\forall k) \\
    & \leq P_{\bar{\tensor{W}},\bar{\bm{\beta}}}(\text{sign}(\hat{\bm{\beta}}_{i_kj_k}^{\mathcal{A}_{k}})=\text{sign}(\bar{\bm{\beta}}_{i_kj_k}^{\mathcal{A}_{k}}),\forall (i_k,j_k) \in \mathcal{A}_{k},\forall k).
\end{align*}
\end{proof}

\begin{proof}[proof of Theorem 2]
Let $\mathcal{E}_{N,k}=\{\text{sign}(\hat{\bm{\beta}}_{i_kj_k}^{\mathcal{A}_{k}})=\text{sign}(\bar{\bm{\beta}}_{i_kj_k}^{\mathcal{A}_{k}})\}$. Then by Theorem 1, $P_{\bar{\tensor{W}},\bar{\bm{\beta}}}(\mathcal{E}_{N,k}) \geq 1-O(\exp(-\eta \log p))$ for large $N$. On $\mathcal{E}_{N,k}$, By the KKT condition and the expansion of $L_{N,\mathcal{A}_{k}}^{\prime}(\bar{\tensor{W}},\hat{\bm{\beta}}^{\mathcal{A}_{k}},\tensor{X})$ at $\bar{\bm{\beta}}^{\mathcal{A}_{k}}$
\begin{align*}
    - \lambda_{N,k} & \text{sign}(\bar{\bm{\beta}}^{\mathcal{A}_{k}}) \\
    & = L_{N,\mathcal{A}_{k}}^{\prime}(\bar{\tensor{W}},\hat{\bm{\beta}}^{\mathcal{A}_{k}},\tensor{X})\\
    & = L_{N,\mathcal{A}_{k}}^{\prime}(\bar{\tensor{W}},\bar{\bm{\beta}}^{\mathcal{A}_{k}},\tensor{X}) + L_{N,\mathcal{A}_{k}\mathcal{A}_{k}}^{\prime\prime}(\bar{\tensor{W}},\bar{\bm{\beta}},\tensor{X}) v_{N,k} \\
    & = \bar{L}^{\prime\prime}_{\mathcal{A}_{k}\mathcal{A}_{k}} v_{N,k}+ L_{N,\mathcal{A}_{k}}^{\prime}(\bar{\tensor{W}},\bar{\bm{\beta}}^{\mathcal{A}_{k}},\tensor{X}) + (L_{N,\mathcal{A}_{k}\mathcal{A}_{k}}^{\prime\prime}(\bar{\tensor{W}},\bar{\bm{\beta}},\tensor{X})-\bar{L}^{\prime\prime}_{\mathcal{A}_{k}\mathcal{A}_{k}}) v_{N,k},
\end{align*}where $v_{N,k}=\hat{\bm{\beta}}^{\mathcal{A}_{k}}-\bar{\bm{\beta}}^{\mathcal{A}_{k}}$. By rearranging the terms
\begin{equation}\label{eqn:thm2_exps1}
\begin{aligned}
    & v_{N,k} = \\
    & -\lambda_{N,k}[\bar{L}^{\prime\prime}_{\mathcal{A}_{k}\mathcal{A}_{k}}]^{-1}\text{sign}(\bar{\bm{\beta}}^{\mathcal{A}_{k}}) - [\bar{L}^{\prime\prime}_{\mathcal{A}_{k}\mathcal{A}_{k}}]^{-1}[L_{N,\mathcal{A}_{k}}^{\prime}(\bar{\tensor{W}},\bar{\bm{\beta}}^{\mathcal{A}_{k}},\tensor{X})+D_{N,\mathcal{A}_{k}\mathcal{A}_{k}}(\bar{\tensor{W}},\bar{\bm{\beta}}^{\mathcal{A}_{k}})v_{N,k}],
\end{aligned} 
\end{equation}where $D_{N,\mathcal{A}_{k}\mathcal{A}_{k}}=L_{N,\mathcal{A}_{k}\mathcal{A}_{k}}^{\prime\prime}(\bar{\tensor{W}},\bar{\bm{\beta}},\tensor{X})-\bar{L}^{\prime\prime}_{\mathcal{A}_{k}\mathcal{A}_{k}}$. Next, for fixed $(i,j) \in \mathcal{A}_{k}^c$, by expanding $L_{N,\mathcal{A}_{k}}^{\prime}(\bar{\tensor{W}},\hat{\bm{\beta}}^{\mathcal{A}_{k}},\tensor{X})$ at $\bar{\bm{\beta}}^{\mathcal{A}_{k}}$
\begin{equation}\label{eqn:thm2_exps2}
    L_{N,ij}^{\prime}(\bar{\tensor{W}},\hat{\bm{\beta}}^{\mathcal{A}_{k}},\tensor{X}) = L_{N,ij}^{\prime}(\bar{\tensor{W}},\bar{\bm{\beta}}^{\mathcal{A}_{k}},\tensor{X}) + L_{N,ij,\mathcal{A}_{k}}^{\prime\prime}(\bar{\tensor{W}},\bar{\bm{\beta}}^{\mathcal{A}_{k}},\tensor{X}) v_{N,k}. 
\end{equation} Then combining \eqref{eqn:thm2_exps1} and \eqref{eqn:thm2_exps2} we get
\begin{equation}\label{eqn:thm2_bouding}
\begin{aligned}
& L_{N,ij}^{\prime}(\bar{\tensor{W}},\hat{\bm{\beta}}^{\mathcal{A}_{k}},\tensor{X}) \\
& = -\lambda_{N,k}\bar{L}^{\prime\prime}_{ij,\mathcal{A}_{k}}(\bar{\bm{\beta}}^{\mathcal{A}_{k}})[\bar{L}^{\prime\prime}_{\mathcal{A}_{k}\mathcal{A}_{k}}]^{-1}\text{sign}(\bar{\bm{\beta}}^{\mathcal{A}_{k}}) - \bar{L}^{\prime\prime}_{ij,\mathcal{A}_{k}}(\bar{\bm{\beta}}^{\mathcal{A}_{k}})[\bar{L}^{\prime\prime}_{\mathcal{A}_{k}\mathcal{A}_{k}}]^{-1}L_{N,\mathcal{A}_{k}}^{\prime}(\bar{\tensor{W}},\bar{\bm{\beta}}^{\mathcal{A}_{k}},\tensor{X}) \\
& + [D_{N,ij,\mathcal{A}_{k}}(\bar{\tensor{W}},\bar{\bm{\beta}}^{\mathcal{A}_{k}})-\bar{L}^{\prime\prime}_{ij,\mathcal{A}_{k}}(\bar{\bm{\beta}}^{\mathcal{A}_{k}})[\bar{L}^{\prime\prime}_{\mathcal{A}_{k}\mathcal{A}_{k}}]^{-1}D_{N,\mathcal{A}_{k}\mathcal{A}_{k}}(\bar{\tensor{W}},\bar{\bm{\beta}}^{\mathcal{A}_{k}})] v_{N,k} \\
& + L_{N,ij}^{\prime}(\bar{\tensor{W}},\bar{\bm{\beta}}^{\mathcal{A}_{k}},\tensor{X}).
\end{aligned}
\end{equation} By the incoherence condition outlined in condition (A3), for any $(i,j) \in \mathcal{A}_{k}$,
\begin{equation*}
    |\bar{L}_{ij,\mathcal{A}_{k}}^{''}(\bar{\tensor{W}},\bar{\bm{\beta}})[\bar{L}_{\mathcal{A}_{k},\mathcal{A}_{k}}^{''}(\bar{\tensor{W}},\bar{\bm{\beta}})]^{-1} \text{sign}(\bar{\bm{\beta}}_{\mathcal{A}_{k}})| \leq \delta < 1.
\end{equation*}Thus, following straightforwardly (with the modification that we are considering each $\mathcal{A}_{k}$ instead of $\mathcal{A}$) from the proofs of Theorem 2 of \citet{peng2009partial}, the remaining terms in \eqref{eqn:thm2_bouding} can be shown to be all $o(\lambda_{N,k})$, and the event $\max_{(i,j) \in \mathcal{A}_{k}^c}|L_{N,ij}^{\prime}(\bar{\tensor{W}},\hat{\bm{\beta}}^{\mathcal{A}_{k}},\tensor{X})| < \lambda_{N,k}$ with probability at least $1-O(\exp(-\eta \log p))$ for sufficiently large $N$. Thus, it has been proved that for sufficiently large $N$, no wrong edge will be included for each true edge set $\mathcal{A}_{k}$ and hence, no wrong edge will be included in $\mathcal{A} = \cup_k \mathcal{A}_{k}$.
\end{proof}

\begin{proof}[proof of Theorem 3]
By Theorem 1 and Theorem 2, with probability tending to $1$, any solution of the restricted problem is also a solution of the original problem. On the other hand, by Theorem 2 and the KKT condition, with probability tending to $1$, any solution of the original problem is also a solution of the restricted problem. Therefore, Theorem 3 follows. 
\end{proof}

\section{Simulated Precision Matrix}
\label{sec:simulated_precision_matrix}
\begin{enumerate}
  \item \textbf{AR1($\rho$)}: The covariance matrix of the form $\mat{A} = (\rho^{|i-j|})_{ij}$ for $\rho \in (0,1)$.
  \item \textbf{Star-Block (SB):} A block-diagonal covariance matrix, where each block's precision matrix corresponds to a star-structured graph with $(\mat{\Psi}_k)_{ij} = 1$. Then, for $\rho \in (0,1)$, we have that $\mat{A}_{ij} = \rho$ if $(i,j) \in E$ and $\mat{A}_{ij}=\rho^2$ for $(i,j) \not \in E$, where $E$ is the corresponding edge set.
  \item \textbf{Erdos-Renyi random graph (ER):} The precision matrix is initialized at $\mat{A} = 0.25 \mat{I}$, and $d$ edges are randomly selected. For the selected edge $(i,j)$, we randomly choose $\psi \in [0.6, 0.8]$ and update $\mat{A}_{ij} = \mat{A}_{ji} \rightarrow \mat{A}_{ij} - \psi$ and $\mat{A}_{ii} \rightarrow \mat{A}_{ii} + \psi$, $\mat{A}_{jj} \rightarrow \mat{A}_{jj} + \psi$.
\end{enumerate}

\end{document}